\documentclass[twoside,11pt,tablecaptionbottom]{jmlr}
\usepackage{booktabs}
\usepackage{array}
\usepackage{multirow}
\usepackage{microtype}
\usepackage{bm}
\usepackage{lastpage}
\usepackage{enumitem}
\usepackage{placeins}
\hypersetup{hidelinks}
\setlist[itemize]{leftmargin=1.6em,itemsep=0.25em,topsep=0.3em}
\setlist[enumerate]{leftmargin=1.8em,itemsep=0.25em,topsep=0.3em}

\jmlryear{2026}
\jmlrpages{1--\pageref{LastPage}}
\makeatletter
\renewcommand*{\@titlefoot}{}
\makeatother
\title[Settling: Equilibrium Inference]{Settling: Equilibrium Inference for Non-Convex Validity Sets}

\author[Saad Saoud]{%
\Name{Lyes Saad Saoud}\\
\addr Independent Researcher,\\
Chicago, IL, USA.}
\begin{document}
\maketitle

\begin{abstract}
Many learning systems return a single point estimate even when admissible outputs form disconnected or non-convex sets. Under squared loss, an ambiguous conditional distribution can therefore have a Bayes-optimal conditional mean that is invalid. We formalize this failure as \emph{conditional mean collapse} and introduce \emph{Settling}, an equilibrium-based inference operator that separates proposal generation, consistency evaluation, and test-time equilibrium selection. The operator treats a mean-seeking proposal as an initialization and refines it toward a locally stable configuration; conditional on initialization, refinement is deterministic. We establish exact-gradient descent, local convergence, and an inexact-gradient robustness condition relevant to learned consistency critics. In a reproducible 100-context geometric diagnostic, the mean-seeking baseline succeeds in 0/100 contexts, stochastic denoising in 100/100, and Settling in 99/100 while producing substantially lower trajectory roughness. A 1,200-run sensitivity study yields 97--100\% success across obstacle-jitter ranges up to 0.20 and 94--100\% across one-time initialization perturbations from $0.05$ to $0.50$. Cross-domain panels remain mechanism illustrations; learned high-dimensional validation remains an open empirical test.
\end{abstract}


\section{Introduction}\label{sec:introduction}
Modern machine-learning systems frequently compress several plausible interpretations into a single representation or decision. Attention and other convex aggregation mechanisms are highly effective when the relevant solution geometry is approximately unimodal or smoothly varying, but a precise failure can arise when mutually valid alternatives occupy a non-convex output set: their average need not itself be valid. This paper studies that inference-level mismatch rather than arguing against attention as a general representation mechanism.

\begin{figure}[t]
\centering
\includegraphics[width=0.98\linewidth]{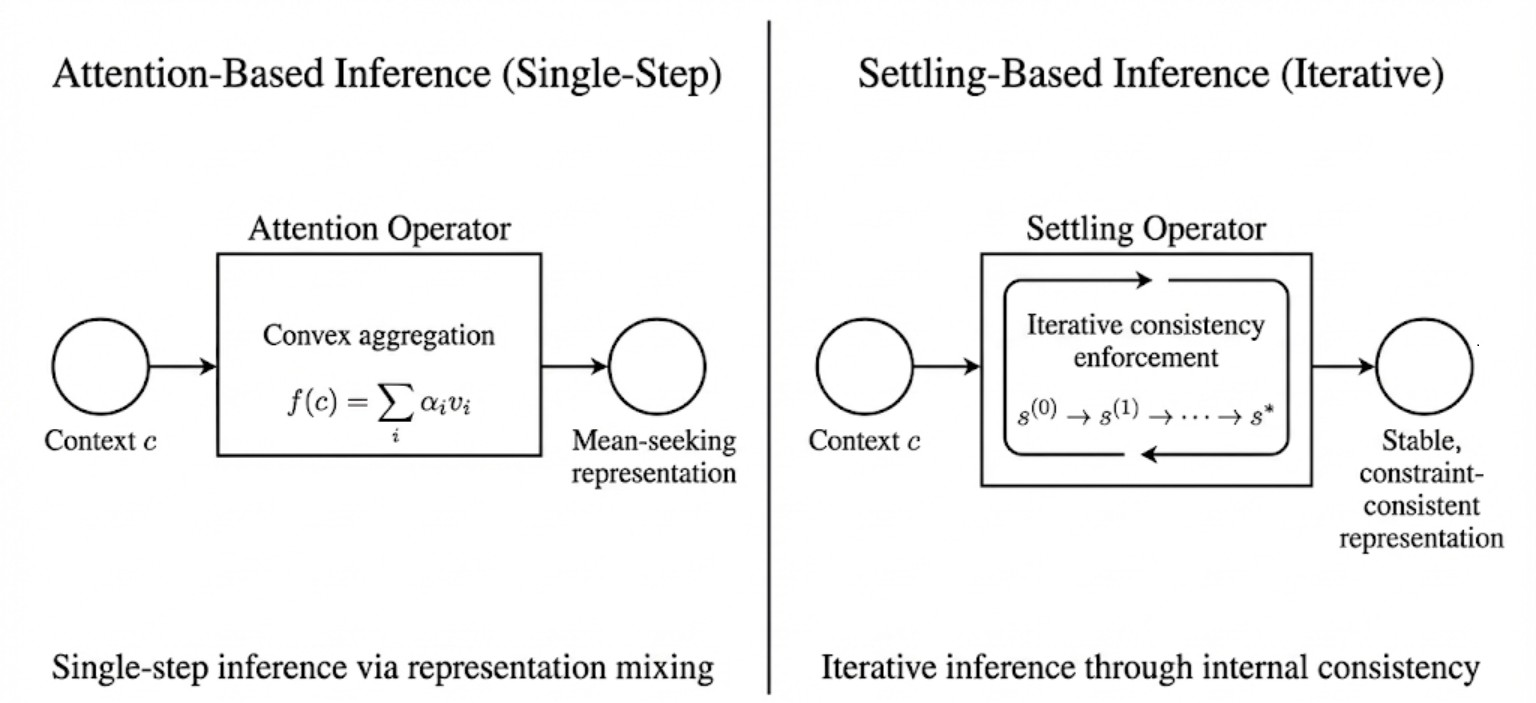}
\caption{Original conceptual comparison retained from the manuscript. Mean-seeking inference produces a single aggregate in one feed-forward step, whereas Settling treats inference as iterative internal consistency enforcement followed by equilibrium selection.}
\label{fig:conceptual-comparison}
\end{figure}

Under squared loss, the Bayes-optimal point predictor is the conditional expectation. If valid outputs form a disconnected or non-convex set, that conditional expectation can lie outside the admissible set even when every mode in the conditional distribution is valid. We call this operator-level failure \emph{conditional mean collapse}. The trajectory example in Fig.~\ref{fig:mean-collapse} makes the geometry explicit: two designated A-to-B alternatives are collision free, while their arithmetic mean passes through the obstacle region.

\begin{figure}[t]
\centering
\includegraphics[width=0.96\linewidth]{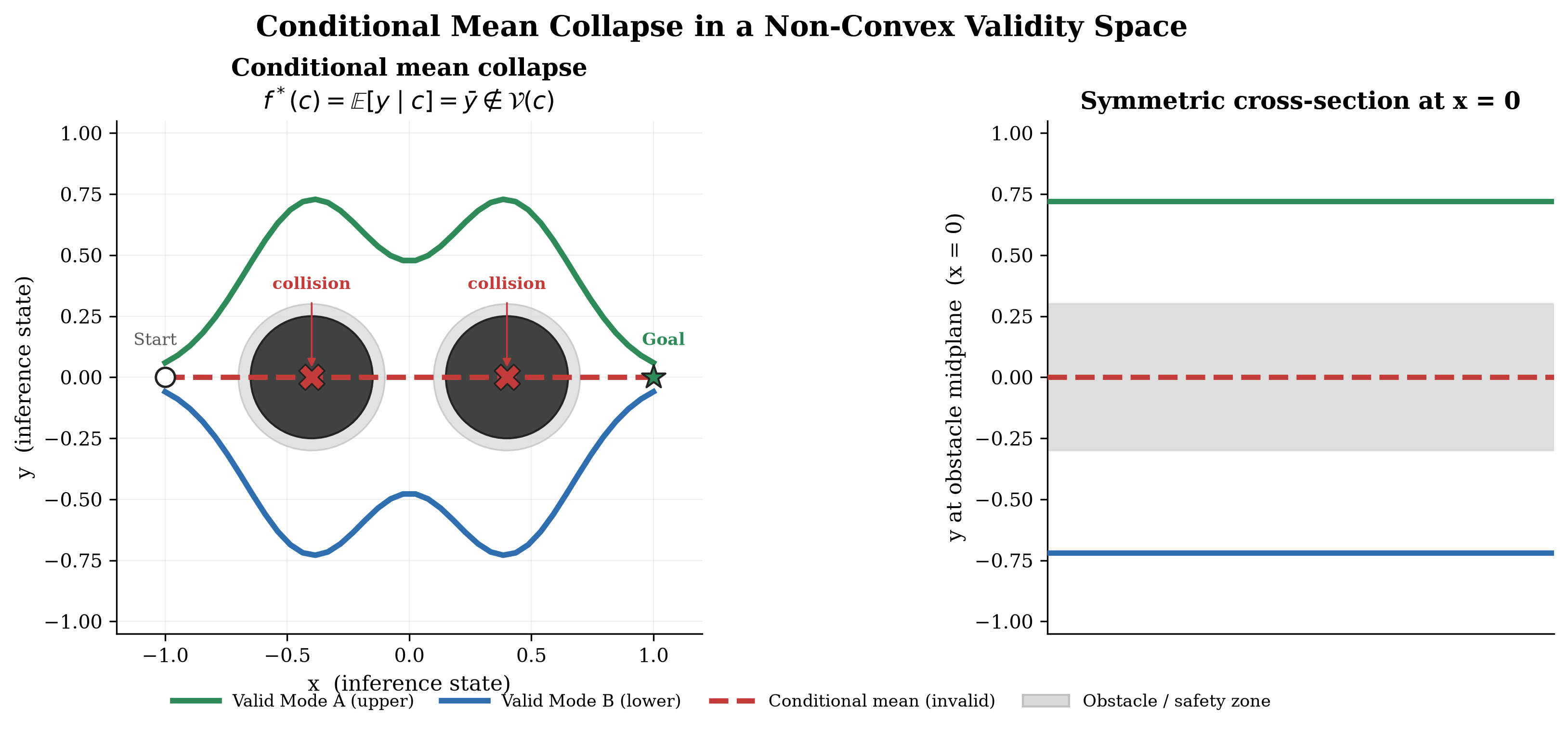}
\caption{Conditional mean collapse in the reproducible geometric diagnostic. Two designated collision-free A-to-B trajectories are individually valid, while their arithmetic mean intersects the obstacles. The result is binary under the stated collision rule: either the complete trajectory clears both obstacles or it does not.}
\label{fig:mean-collapse}
\end{figure}

We introduce \emph{Settling} as a complementary inference paradigm. Rather than treating a single-step aggregate as the final answer, Settling uses an amortized proposal only to initialize a dynamical process. A consistency landscape then rejects invalid configurations and drives the state toward a locally stable equilibrium. The intended distinction is procedural: the contribution is not a new optimizer, a new density model, or a claim that all attention-based systems must return invalid means. Instead, we identify a precise setting in which mean-seeking point inference is geometrically mismatched and study an explicit equilibrium-selection operator for that setting.

The paper makes four contributions. First, it formalizes conditional mean collapse for squared-loss point estimation in non-convex validity sets and states the precise conditions under which a Bayes-optimal point estimate is invalid. Second, it defines Settling as an inference-operator decomposition, consisting of proposal generation, consistency evaluation, and equilibrium selection, and establishes local descent and convergence properties under explicit smoothness and basin assumptions. Third, it separates this operator interpretation from energy-based density modeling, stochastic diffusion, deep equilibrium networks, structured energy prediction, and classical trajectory optimization without claiming novelty for gradient descent or fixed-point computation themselves. Fourth, it provides a fully reproducible 100-context geometric mechanism study, along with a 1,200-run robustness sweep over context perturbations and symmetry-breaking amplitudes, including source code, publication figures, quantitative metrics, and an inference animation, while explicitly distinguishing reproducible evidence from analytical or legacy cross-domain illustrations.

\paragraph{Validation scope.}
The analytic study is designed as an identification experiment for the inference mechanism: proposal geometry, validity, and the consistency landscape are all known, so failure or recovery can be attributed to the inference operator rather than to representation error, training instability, or hidden data effects. Section~\ref{sec:training} specifies a compatible learnable realization, while Section~\ref{sec:theory_settling} gives local conditions under which approximate consistency gradients preserve descent. The analytic setting is therefore a controlled unit test of the mechanism, not a surrogate claim that learned high-dimensional deployment has already been established.

\paragraph{Conditions for learned transfer.}
The theory identifies three local conditions that a future learned realization would need to satisfy for the same mechanism to be expected: (i) the proposal enters a recoverable neighborhood of a valid basin or an unstable compromise between basins; (ii) the learned consistency gradient remains sufficiently aligned with a validity-improving direction; and (iii) the refinement budget reaches a stable state at acceptable test-time cost. These conditions are not asserted to hold for arbitrary learned models and are not used as empirical evidence in this paper. Their purpose is to separate the operator-level mechanism established here from the distinct representation-learning problem that a learned deployment must solve. Lemma~\ref{lem:inexact-gradient} formalizes condition (ii) for bounded relative gradient error.

The evaluation is deliberately controlled. The purpose is not leaderboard performance or a claim of state-of-the-art planning, vision-language reasoning, or sensor fusion. The purpose is to isolate the geometry of ambiguity and ask whether an inference operator can reject an invalid aggregate and reach a constraint-consistent equilibrium. This distinction is especially important for interpreting the geometric success rates: they are finite-sample outcomes under the declared binary A-to-B criterion, not universal safety guarantees. The current released implementation obtains 99/100 collision-free Settling trajectories in the nominal 100-context evaluation described below.

\section{Failure of Mean-Seeking Inference}
\label{sec:failure}

We analyze a structural limitation shared by a broad class of contemporary inference operators when applied in non-convex solution spaces.
Our focus is deliberately narrow: we isolate the failure to the \emph{inference operator itself}, independent of model capacity, training data, architectural depth, or optimization procedure.
We show that standard point-estimation objectives induce mean-seeking behavior under ambiguity, which leads to invalid or infeasible predictions whenever the set of valid solutions is non-convex or disconnected.
We refer to this operator-level failure mode as \emph{conditional mean collapse}.

\subsection{Optimality of the Conditional Mean}

Consider the task of learning a mapping from a context space $\mathcal{C}$ to a continuous target space $\mathcal{Y} \subseteq \mathbb{R}^d$.
Let $f_\theta : \mathcal{C} \rightarrow \mathcal{Y}$ be a parametric predictor trained under the expected $L_2$ risk
\begin{equation}
\mathcal{L}(\theta) = \mathbb{E}_{p(c,y)} \left[ \| y - f_\theta(c) \|^2_2 \right].
\end{equation}
A classical result in statistical decision theory states that the Bayes-optimal predictor minimizing this risk is the conditional expectation~\cite{bishop2006pattern}
\begin{equation}
    f^*(c) = \mathbb{E}_{p(y|c)}[y]
    = \int_{\mathcal{Y}} y \, p(y|c) \, dy.
    \label{eq:cond_mean}
\end{equation}

Squared loss provides the cleanest formal case because its Bayes estimator is uniquely the conditional mean. Other point-estimation objectives can exhibit related failures, but their Bayes estimators need not equal the arithmetic mean in every multimodal setting. We therefore use squared loss for the formal theorem and treat broader central-estimator behavior as a motivation rather than a universal claim. The structural issue is the mismatch between a single convexly aggregated estimate and a non-convex validity set.

We formalize this limitation as a structural property of point-estimation inference.

\begin{theorem}[Conditional Mean Collapse under squared loss]
\label{thm:loss_change_not_enough}
Assume there exists a context $c$ and two valid outputs $y_1,y_2\in\mathcal{V}(c)$ such that
$\bar y=\lambda y_1+(1-\lambda)y_2\notin\mathcal{V}(c)$ for some $\lambda\in(0,1)$.
For the ambiguous conditional distribution
$p(y\mid c)=\lambda\delta(y-y_1)+(1-\lambda)\delta(y-y_2)$,
the Bayes-optimal point predictor under squared loss is the conditional mean $\bar y$ and is therefore invalid.
Hence, increasing model capacity without changing the mean-seeking inference objective cannot remove this failure for the stated context.
\end{theorem}

\noindent\textit{Proof provided in Appendix~\ref{app:proofs}.}

Set-valued, mixture-valued, or explicitly mode-seeking predictors can preserve multiple alternatives and are not ruled out by Theorem~\ref{thm:loss_change_not_enough}. Our claim is narrower: a single squared-loss point estimate can be Bayes-optimal while lying outside a non-convex valid set.

\subsection{Topological Mismatch Under Ambiguity}

Let $\mathcal{V}(c)\subseteq\mathcal{Y}$ denote the externally defined set of admissible outputs for context $c$. We do not identify validity with the support of the observed conditional distribution: valid outputs may be unobserved, and observed outputs may be noisy. The relevant failure occurs when the conditional distribution places mass on valid alternatives whose convex combination lies outside $\mathcal{V}(c)$.

Consider the symmetric two-mode case
\begin{equation}
    p(y|c) = \tfrac{1}{2}\delta(y - \mu_1) + \tfrac{1}{2}\delta(y - \mu_2),
    \qquad \mu_1,\mu_2\in\mathcal{V}(c).
\end{equation}
Under squared loss, the Bayes-optimal point estimate is
\begin{equation}
f^*(c) = \tfrac{1}{2}(\mu_1 + \mu_2).
\end{equation}
If the segment between $\mu_1$ and $\mu_2$ crosses a forbidden region, then $f^*(c)\notin\mathcal{V}(c)$. Perfect symmetry is not required for the general theorem: for unequal mode weights, the squared-loss Bayes estimator is the corresponding convex interpolation, and invalidity follows whenever that interpolation leaves the admissible set. Thus the core issue is not symmetry itself but the mismatch between a point-estimation statistic and non-convex validity geometry.

\subsection{Convex Aggregation in a Restricted Attention Setting}

The conditional-mean result is objective-level and does not imply that every Transformer or attention-based system must return an invalid average. A narrower geometric observation is nevertheless useful. For a standard scaled dot-product attention head,
\begin{equation}
    h_i = \sum_{j=1}^T \mathrm{softmax}(q_i^\top k_j)_{ij} \, v_j,
    \qquad
    h_i \in \mathrm{Conv}(\{v_j\}),
\end{equation}
so the head output is a convex combination of its value vectors. If those vectors directly parameterize mutually valid alternatives and the downstream decision is a single-step mean-seeking map, the aggregate can fall in a convex-hull region that is invalid in output space. This is the restricted attention-like setting represented by the analytical mean baseline in Evidence~I.

This observation should not be generalized to complete Transformer architectures without qualification. Multi-layer nonlinear transformations, discrete decoding, mixture heads, search, or explicit constraint modules can map an internal convex combination to a non-convex output set and can preserve or recover multimodality. Our claim is therefore not that attention is intrinsically incapable of non-convex reasoning. The claim is that a single convex aggregate used as a terminal point estimate provides no intrinsic mechanism for rejecting an invalid compromise. Settling addresses precisely that restricted operator-level gap by making the aggregate revisable at inference time.

\section{Settling as an Inference Paradigm}
\label{sec:settling_paradigm}

To address the structural limitations of mean-seeking inference, we reformulate neural inference as a \emph{dynamical resolution process} rather than a single-step functional evaluation.
Instead of producing predictions through feed-forward aggregation, \emph{Settling} defines inference as the evolution of an internal representation toward a stable, constraint-consistent equilibrium.
In this formulation, inference is not the computation of an output statistic, but the resolution of ambiguity through internally consistent dynamics.

This perspective differs from single-pass point prediction at the level of the inference operator.
Many common point-estimation pipelines terminate after producing a single context-conditioned state, whether or not that state lies in an explicitly valid region.
Settling instead treats that state as revisable and applies structured relaxation over a context-conditioned consistency landscape, analytic or learned, enabling local mode selection in non-convex and ambiguous environments.
Importantly, Settling does not replace existing architectures or increase representational capacity.
It modifies how inference is performed given a representation.
The distinction is therefore procedural rather than architectural.

\subsection{Inference as Dynamical Relaxation}

Settling represents validity through a context-dependent scalar potential, which may be analytic or learned,
\begin{equation}
\mathcal{E} : \mathcal{S} \times \mathcal{C} \rightarrow \mathbb{R},
\end{equation}
where low-energy regions correspond to internally consistent, constraint-satisfying configurations.
Rather than explicitly predicting an optimal output, inference allows the internal representation to evolve until inconsistencies are eliminated.

Inference is governed by a recurrent \emph{inference operator}
\begin{equation}
\mathcal{T} : \mathcal{S} \rightarrow \mathcal{S},
\end{equation}
which iteratively updates the internal state conditioned on the context.
Starting from an initial hypothesis $s^{(0)}$, inference proceeds via repeated application of $\mathcal{T}$,
\begin{equation}
s^{(k+1)} = \mathcal{T}(s^{(k)}),
\end{equation}
until convergence to a fixed point
\begin{equation}
s^\star = \lim_{k \to \infty} s^{(k)},
\qquad
\nabla_s \mathcal{E}(s^\star, c) = 0.
\end{equation}

Crucially, $\mathcal{E}$ is not an externally specified task objective and does not encode a likelihood, reward, or cost to be globally minimized.
It is a context-conditioned consistency functional, analytic or learned, that encodes semantic, geometric, or physical constraints.
Settling therefore defines an \emph{inference operator}, not a task-specific optimizer.
The objective is not global optimality, but convergence to a locally stable equilibrium that resolves ambiguity under the given context.

We make no assumption of global convexity, contraction, or uniqueness of equilibria.
Stability is inherently local and conditioned on the structure of the chosen consistency landscape, whether analytic or learned.
As a result, inference dynamics naturally partition the state space into basins of attraction corresponding to distinct valid interpretations.

\subsection{Symmetry Breaking and Equilibrium Selection}

In ambiguous or bifurcated settings, the energy landscape typically contains multiple low-energy basins corresponding to distinct valid solutions.
Let $\mu_1$ and $\mu_2$ denote two such configurations, and let $\bar{\mu}$ denote their arithmetic mean.
In non-convex validity sets, the mean configuration is generally unstable:
\begin{equation}
\mathcal{E}(\mu_1) \approx \mathcal{E}(\mu_2) \ll \mathcal{E}(\bar{\mu}).
\end{equation}

Expectation-seeking inference operators collapse toward $\bar{\mu}$, as discussed in Section~\ref{sec:failure}.
Settling dynamics, in contrast, do not aggregate across hypotheses.
Instead, ambiguity is resolved through time evolution.
Unstable saddle points repel trajectories, while locally stable equilibria attract them.

Inference therefore replaces \emph{aggregation} with \emph{selection}.

The selected equilibrium is deterministic conditional on the initialization, energy landscape, update rule, and step size. In exactly symmetric configurations, a one-time initialization perturbation may be used only as a tie breaker; it is not injected throughout refinement and is not used to sample an output distribution. Once initialized, the settling trajectory is deterministic. This conditional notion of determinism is used throughout the paper and is distinct from end-to-end determinism under a freshly randomized initialization.

In landscapes with a symmetric unstable equilibrium, the same dynamics can produce a nonlinear symmetry-breaking transition.
As ambiguity varies, a mean-seeking point estimate may change continuously toward a compromise, whereas equilibrium-selection dynamics can change branches when the symmetric state loses local stability.
We use \emph{bifurcation} only in this dynamical-systems sense; we do not claim a thermodynamic phase transition or universal critical scaling.
The analytical reference in Section~\ref{subsec:exp-semantic} is therefore a normal-form illustration of branch selection, not empirical evidence of a universal threshold.

\subsection{Modular Structure of Settling Inference}

Settling decomposes inference into three irreducible components:
(i) hypothesis generation,
(ii) constraint evaluation,
and (iii) equilibrium selection.
This decomposition is not an architectural convenience, but a structural requirement for inference in non-convex validity sets.

\paragraph{Module I: Proposal Network (Hypothesis Generation).}
The proposal module is a fast, amortized mapping
\begin{equation}
\pi_\phi : \mathcal{C} \rightarrow \mathcal{S},
\end{equation}
which produces an initial hypothesis
\begin{equation}
s^{(0)} = \pi_\phi(c).
\end{equation}
Its role is semantic grounding and computational efficiency, not correctness.
Under ambiguity, $\pi_\phi$ may converge toward the conditional mean, which may lie in an invalid or unstable region of the state space.
Settling explicitly tolerates this behavior, as inference does not terminate at the proposal.

\paragraph{Module II: Energy Function (Constraint Evaluation).}
The energy function
\begin{equation}
E : \mathcal{S} \times \mathcal{C} \rightarrow \mathbb{R}
\end{equation}
assigns low energy to configurations that satisfy learned or analytic constraints.
The valid set is defined implicitly as
\begin{equation}
\mathcal{V}(c) = \{ s \in \mathcal{S} \mid E(s, c) \le \varepsilon \}.
\end{equation}
Unlike feed-forward predictors, the energy function provides an explicit mechanism for hypothesis rejection.
It can represent non-convex validity sets that are inaccessible to convex aggregation or single-step inference.

\paragraph{Module III: Settling Dynamics (Equilibrium Selection).}
Inference proceeds through iterative state evolution,
\begin{equation}
s^{(k+1)} = s^{(k)} - \eta \nabla_s E(s^{(k)}, c),
\end{equation}
initialized at $s^{(0)} = \pi_\phi(c)$.
Iteration continues until convergence to a stationary configuration
\begin{equation}
s^\star
\quad \text{s.t.} \quad
\nabla_s E(s^\star, c) = 0.
\end{equation}
Rather than computing a statistic of the conditional distribution, settling deterministically selects a single, internally consistent equilibrium.

\paragraph{Necessity of the Decomposition.}
Each component is essential.
Without the proposal network, inference reduces to unconstrained global search.
Without the energy function, the system collapses to mean-seeking approximation.
Without settling dynamics, invalid hypotheses cannot be corrected.
Only their combination enables deterministic, constraint-consistent inference and structured mode selection in non-convex solution spaces.
\section{A Practical Instantiation of Settling Inference}
\label{sec:method}

We now describe a concrete instantiation of the settling paradigm introduced in Section~\ref{sec:settling_paradigm}.
This instantiation operationalizes settling as a modular inference operator composed of hypothesis generation, constraint evaluation, and equilibrium selection.
We refer to this realization as \emph{Settling-Based Inference} (SBI).
SBI does not introduce a new inference principle.
Rather, it provides a realizable mechanism for equilibrium-based inference under ambiguity.
Alternative architectural realizations are possible, provided they preserve the same operator-level structure and inference dynamics.

Crucially, the behavior of SBI does not arise from optimization alone.
The settling dynamics act as an \emph{inference operator} applied at test time to ambiguous inputs, where the initial proposal is intentionally mean-seeking rather than validity-preserving.
The role of the dynamics is not to solve a task-defined objective, but to resolve inconsistency between a proposal and a context-conditioned consistency structure during inference.

Operationally, SBI can be viewed as two coupled stages.
The first stage constructs a stable internal configuration that satisfies contextual constraints without committing to a discrete interpretation.
The second stage resolves residual ambiguity by inducing structured competition among candidate interpretations, resulting in deterministic equilibrium selection.

\begin{figure}[!htbp]
\centering
\includegraphics[width=0.92\linewidth]{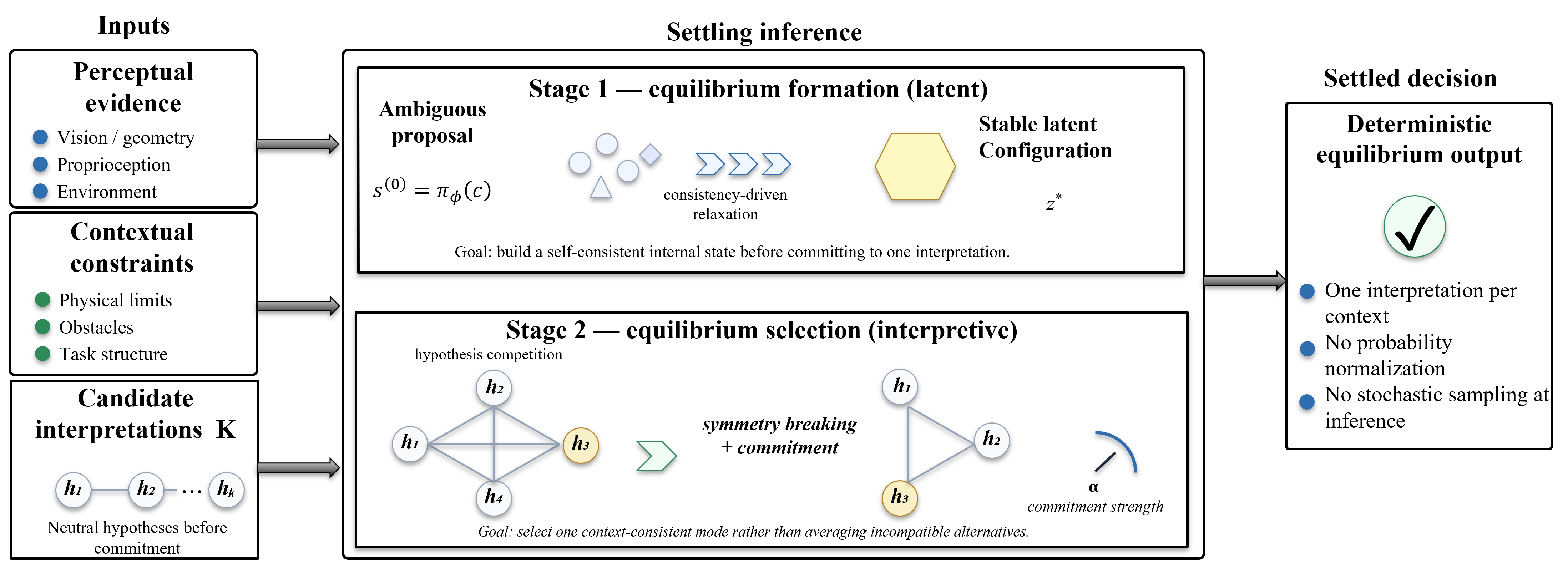}
\caption{SBI. Perceptual evidence, contextual constraints, and candidate interpretations initialize an ambiguous proposal. Stage~1 forms a self-consistent latent configuration; Stage~2 performs symmetry breaking and equilibrium selection, producing one committed equilibrium for the given initialized state. The schematic clarifies the operator decomposition and is not an empirical result.}
\label{fig:sbi-architecture}
\end{figure}
\FloatBarrier

\subsection{Energy Composition and Consistency Modeling}

SBI relies on a scalar consistency functional that evaluates whether an internal configuration is compatible with perceptual evidence and contextual constraints.
In practice, this functional is implemented as a composite energy
\begin{equation}
\mathcal{E}_{\text{total}}(s, c)
=
E_{\text{constraint}}(s, c)
+ \lambda_p \| s - \pi_\phi(c) \|_2^2
+ \lambda_s \mathcal{R}_{\text{smooth}}(s),
\end{equation}
where each term serves a distinct inferential role rather than representing a probabilistic likelihood or a task-defined objective.

The \textbf{constraint energy} $E_{\text{constraint}}(s,c)$ assigns low values to configurations that satisfy learned or analytic constraints imposed by the environment, task structure, or physical feasibility.
This term defines a generally non-convex validity structure and provides an explicit mechanism for rejecting invalid configurations.
The \textbf{proposal anchoring term} $\| s - \pi_\phi(c) \|_2^2$ anchors inference to the amortized proposal produced by the hypothesis generator.
This term does not encode uncertainty or a Bayesian prior.
Instead, it enforces semantic coherence between the evolving internal configuration and the context-conditioned proposal, preventing convergence to constraint-satisfying but context-irrelevant equilibria.
Finally, the \textbf{structural regularization} term $\mathcal{R}_{\text{smooth}}(s)$ encodes generic preferences such as smoothness or sparsity.
It shapes the geometry of the settling trajectory without, by itself, resolving ambiguity or enforcing correctness.

\paragraph{Energy Alignment Assumption.}
Settling assumes that locally valid configurations correspond to lower energy than invalid ones.
This assumption is standard in constraint-driven inference systems and energy-based formulations.
Our analysis does not rely on global optimality or convexity of the energy landscape.
Instead, it requires only the existence of locally stable equilibria corresponding to valid interpretations.
Settling is therefore concerned with inference-time collapse under ambiguity, rather than with guarantees of global optimality.

\subsection{Settling Dynamics as an Inference Operator}

Inference in SBI is realized through iterative state evolution governed by the consistency landscape
\begin{equation}
s^{(k+1)} =
\mathrm{Proj}_{\Omega}
\left(
s^{(k)} - \eta \nabla_s \mathcal{E}_{\text{total}}(s^{(k)}, c)
\right),
\qquad
\\
s^{(0)} = \pi_\phi(c) + \epsilon,
\end{equation}
where $\Omega$ denotes optional simple hard domain constraints and $\epsilon$ is a one-time symmetry-breaking perturbation used only when an initialization lies at an exactly symmetric saddle.

\textbf{Role of Projection.}
The projection operator $\mathrm{Proj}_{\Omega}$ enforces only simple domain bounds.
It does not encode complex validity constraints such as obstacles, semantic consistency, or task logic.
All non-trivial validity structure is learned and enforced through the constraint energy $E_{\text{constraint}}$.
Avoidance of invalid configurations therefore emerges from the dynamics of the consistency landscape rather than from the projection step.

These dynamics are not executed to solve an externally specified optimization problem.
Rather, they implement inference by progressively resolving inconsistency between the internal configuration and the context-conditioned consistency structure.
Iteration continues until convergence to a stable fixed point corresponding to a committed interpretation that satisfies validity constraints under the given context.

A perturbation is introduced only to destabilize an exactly symmetric saddle. No further noise is injected by Settling. Conditional on the initialized state, the subsequent dynamics evolve deterministically toward a stable equilibrium; no iterative stochastic sampling or probability normalization is performed.

\subsection{Necessity of the Three Components}

SBI decomposes inference into proposal generation, constraint evaluation, and settling dynamics.
Each component plays a necessary role.

\begin{enumerate}
    \item \textbf{Proposal alone is insufficient.}
    A purely feed-forward predictor collapses toward the conditional mean under ambiguity.
    In non-convex validity sets, this produces internally inconsistent or infeasible representations, as shown in Section~\ref{sec:failure}.
    
    \item \textbf{Constraint evaluation alone is insufficient.}
    Evaluating constraints without an amortized proposal reduces inference to unconstrained non-convex search, which is computationally inefficient and prone to convergence toward context-irrelevant equilibria.
    
    \item \textbf{Dynamics alone are insufficient.}
    Without anchoring to the proposal, settling dynamics may converge to constraint-satisfying configurations that are disconnected from the input context.
\end{enumerate}

SBI emerges from the interaction of all three components.
The proposal provides rapid, context-sensitive hypothesis generation.
The energy function defines a non-convex validity structure.
The dynamics perform deterministic equilibrium selection.
Together, these elements enable structured mode selection in settings where single-step, mean-seeking inference collapses internal structure.

\section{Theoretical Properties of Settling}
\label{sec:theory_settling}

Having formalized settling as an inference operator rather than a feed-forward mapping, we now characterize its fundamental theoretical properties.
Our analysis focuses on four properties relevant to the operator interpretation:
(i) monotonic improvement under exact consistency gradients,
(ii) robustness of descent to bounded inexact gradients,
(iii) local convergence to stable equilibria under explicit regularity assumptions,
and (iv) deterministic equilibrium selection conditional on initialization.

Throughout this section, we emphasize that these properties concern \emph{inference-time dynamics} induced by the settling operator.
They do not rely on convexity, global optimality, or probabilistic correctness.
All results are local and conditional on the structure of the consistency landscape, whether analytic or learned.
This scope is intentional and reflects the role of settling as an inference mechanism rather than as an optimization algorithm or a probabilistic estimator.

\subsection{Monotonic Improvement of Internal Consistency}

Unlike a terminal single-step point estimate, Settling exposes a test-time sequence of states whose consistency with an explicit landscape can be evaluated at each iteration.
This supplies a notion of \emph{computational progress} tied to the same validity-oriented functional used for refinement.

\begin{lemma}[Monotonic Consistency Improvement]
\label{lem:lyapunov}
Let $\mathcal{E}(s,c)$ be a continuously differentiable consistency functional with $L$-Lipschitz continuous gradients and bounded below.
Then, for any step size $\eta < 2/L$, the settling update
$
s^{(k+1)} = s^{(k)} - \eta \nabla_s \mathcal{E}(s^{(k)}, c)
$
induces a discrete-time dynamical process in which $\mathcal{E}(s^{(k)},c)$ decreases monotonically with each iteration.
\end{lemma}

\noindent\textit{Proof provided in Appendix~\ref{app:monotonicity}.}

This result establishes that, for the idealized exact-gradient operator, the consistency value supplies an ordered test-time progress measure.
The quantity being decreased is the inference-time consistency functional rather than the training loss. A single-pass point predictor may of course contain many internal layers, but it does not by itself expose a test-time sequence ordered by this same validity-oriented functional.

\subsection{Robustness to Inexact Consistency Gradients}
A learned consistency critic need not reproduce an analytic energy exactly; what matters locally is whether its update direction remains sufficiently aligned with the gradient of an ideal consistency functional. The following standard inexact-gradient bound makes that requirement explicit.

\begin{lemma}[Inexact Consistency Descent]
\label{lem:inexact-gradient}
Let $\mathcal{E}(s,c)$ be $L$-smooth and let the implemented update use
\begin{equation}
\tilde g(s,c)=\nabla_s\mathcal{E}(s,c)+e(s,c),
\end{equation}
with relative error $\lVert e(s,c)\rVert_2\leq\delta\lVert\nabla_s\mathcal{E}(s,c)\rVert_2$ for some $0\leq\delta<1$. Then the update
\begin{equation}
s^+=s-\eta\tilde g(s,c)
\end{equation}
strictly decreases $\mathcal{E}$ whenever $\nabla_s\mathcal{E}(s,c)\neq0$ and
\begin{equation}
0<\eta<\frac{2(1-\delta)}{L(1+\delta)^2}.
\end{equation}
\end{lemma}

\noindent\textit{Proof provided in Appendix~\ref{app:inexact-gradient}.}

Lemma~\ref{lem:inexact-gradient} is a local transfer statement, not a learnability theorem. It shows that exact reproduction of the analytic consistency field is unnecessary for monotone refinement: sufficiently small relative gradient error preserves descent. It does not guarantee that a learned critic will satisfy the bound globally, place minima at semantically correct states, or avoid spurious basins; those remain empirical questions.

\subsection{Convergence to Stable Equilibria}

Settling resolves ambiguity locally by converging toward stable configurations of the consistency landscape, including settings in which the valid set is non-convex or disconnected.

\begin{theorem}[Local Convergence to a Stable Valid Equilibrium]
\label{thm:settling_convergence}
Fix a context $c$ and let $s^\star\in\mathcal{V}(c)$ be a stationary point of $\mathcal{E}(\cdot,c)$. Suppose there exists a convex, forward-invariant neighborhood $U$ of $s^\star$ on which $\mathcal{E}(\cdot,c)$ is $\mu$-strongly convex and has $L$-Lipschitz gradient, with $0<\mu\leq L$. Then for any $s^{(0)}\in U$ and step size $0<\eta\leq1/L$, the unprojected Settling update converges linearly to $s^\star$.
\end{theorem}

\noindent\textit{Proof provided in Appendix~\ref{app:local-convergence}.}

\noindent\textit{Scope Note.}
The theorem is deliberately local. A globally non-convex landscape may contain several valid or spurious basins, and the result does not state which basin an arbitrary initialization will enter. It states that once the initialized state lies in a forward-invariant neighborhood with ordinary local regularity, the refinement converges to the corresponding stable valid equilibrium. If a learned consistency critic is misspecified so that its stable point is not valid, the theorem does not convert that point into a correct one.

\subsection{Deterministic Equilibrium Selection Under Ambiguity}

A defining property of settling inference is its ability to resolve multimodality without resorting to stochastic sampling or expectation-based aggregation.

\paragraph {Deterministic Equilibrium Selection}
Settling does not require uniqueness of equilibria. In multimodal settings, the equilibrium reached is determined by the initialization and by the local geometry of the consistency landscape. Conditional on that initialization, settling implements deterministic instance-level equilibrium selection rather than averaging incompatible hypotheses or performing iterative stochastic sampling.

This behavior contrasts with mean-seeking point inference, which can collapse distinct valid modes into an invalid average, and with sampling-based approaches that represent ambiguity through multiple stochastic outcomes.
Settling instead exploits the stability structure of an explicit consistency landscape to produce a single, internally consistent representation per context.
Ambiguity resolution therefore emerges from inference-time dynamics rather than from probabilistic normalization or stochastic exploration.
\section{Relation to Existing Paradigms}
\label{sec:relations}

Settling shares computational machinery with energy-based models, diffusion models, and classical optimization or planning methods.
These similarities arise at the level of algorithmic components rather than inferential intent.
Settling is formulated as an \emph{inference operator} for deterministic equilibrium selection under ambiguity.
It is not a generative model, not a density estimator, and not a claim of a new generic optimization algorithm.

This section situates settling relative to existing approaches by clarifying the role inference plays in each paradigm.

\subsection{Relation to Energy-Based Models}

Energy-Based Models (EBMs) use learned energy functions to score configurations and, in probabilistic formulations, may define unnormalized densities such as
\begin{equation}
p(x) \propto \exp(-E_\theta(x)).
\end{equation}
Depending on the formulation, inference may use sampling, MAP-style optimization, or other structured procedures~\cite{lecun2006tutorial,du2019implicit}. Settling therefore differs primarily in inferential role rather than in the mere presence of an energy function.
The energy function in settling is not trained to model a probability density and does not define a generative process.
Instead, it serves as a \emph{consistency critic}, assigning low energy to constraint-satisfying configurations and higher energy to invalid ones under a fixed context.
No normalization, marginalization, or sampling is performed.
The energy landscape is used solely to shape inference-time dynamics.

Settling uses its consistency landscape specifically as a test-time selection mechanism: given a context and an initial hypothesis, refinement converges, conditional on initialization and local basin geometry, toward a single stable equilibrium. This interpretation is narrower than the broader probabilistic and structured-prediction uses of energy functions.

\subsection{Relation to Diffusion and Score-Based Models}

Diffusion and score-based models generate data by reversing a stochastic noise process, progressively transforming noise into structured samples~\cite{ho2020denoising}.
These methods are explicitly designed for sampling and distributional coverage.

Although settling also employs iterative refinement, its purpose is fundamentally different.
Settling is not a sampling procedure and does not aim to represent distributional diversity.
Instead, it performs deterministic relaxation from an initial hypothesis toward a stable equilibrium under learned constraints.

Stochasticity plays distinct roles.
In diffusion models, noise is essential throughout inference.
In Settling, noise, when used at all, is confined to initialization to destabilize symmetric saddle points. Once initialized away from exact symmetry, inference proceeds deterministically toward a single equilibrium for that initialized state.

\subsection{Relation to Optimization and Planning}

Classical planning and trajectory optimization methods minimize explicitly defined cost functionals to compute feasible solutions~\cite{ratliff2009chomp,schulman2014trajopt}.
These approaches assume that objectives and constraints are specified a priori.

Settling can be viewed as inference-time relaxation over a context-conditioned consistency landscape that may be analytic or learned.
Unlike classical optimization, settling does not solve a user-defined task objective.
The dynamics are a \emph{mechanism}, not the goal: they resolve inconsistency between an amortized proposal and learned validity constraints.
Settling therefore reframes optimization-style dynamics as an inference operator rather than as a problem solver.

\subsection{Relation to Deep Equilibrium and Structured Energy Models}
Deep equilibrium models compute fixed points of learned transformations and use implicit differentiation to train effectively infinite-depth networks~\cite{bai2019deep}. Structured prediction energy networks use learned energies to score structured outputs and perform optimization-based prediction~\cite{belanger2016structured}, while differentiable optimization layers embed optimization problems inside neural networks~\cite{amos2017optnet}. These lines of work establish the value of fixed-point computation and energy-shaped prediction, but they address different questions. Settling is introduced here as an inference operator specifically for ambiguity in non-convex validity sets: an amortized proposal may be invalid, and test-time dynamics are used to reject that proposal and select a locally stable, constraint-consistent equilibrium. We therefore do not claim novelty for fixed-point iteration or gradient-based relaxation themselves; the contribution is the conditional-mean-collapse formulation, the operator-level decomposition, and the resulting inference interpretation.

\subsection{Structural Comparison of Inference Paradigms}

The paradigms differ primarily in inferential purpose. Mean-seeking aggregation produces a single aggregate; diffusion emphasizes stochastic sampling and distributional coverage; conventional EBMs define energy-shaped densities and often rely on sampling or optimization; Settling uses an energy-like consistency landscape to refine one initialized hypothesis toward an equilibrium. The comparison concerns inference semantics rather than task-level superiority.

The restricted attention-like baseline studied here uses a terminal convex aggregate, which can be misaligned with a non-convex validity set when that aggregate lies outside the admissible region.
Diffusion models prioritize stochastic sampling for distributional coverage.
EBMs represent densities and typically rely on sampling-based inference.

Settling departs from these paradigms by redefining the inference objective itself.
Rather than computing expectations, sampling distributions, or estimating densities, Settling refines an initialized state toward an equilibrium of an explicit consistency landscape.
The output is a stable fixed point of a context-conditioned dynamical system, making constraint satisfaction explicit rather than emergent.

\section{Analytic Validation and a Learnable Extension}
\label{sec:training}

\subsection{What Is Validated in This Paper}
The quantitative contribution of this paper concerns the \emph{inference operator}: an initially mean-seeking hypothesis is evaluated against a non-convex consistency landscape and refined toward a locally stable valid equilibrium. In the reproducible study, both the proposal and the consistency energy are analytic. This choice is deliberate because it makes the validity geometry observable and isolates inference-time equilibrium selection from representation learning, model capacity, and training instability.

Accordingly, this paper does \emph{not} empirically establish that a learned proposal network and a learned energy network will preserve the same behavior in high-dimensional data. Such a result would require a separate evaluation of representation quality, negative-sample construction, energy misspecification, optimization stability, scaling, and compute. The learned formulation below is therefore a compatible extension of the operator decomposition, not an additional empirical claim.

\subsection{Possible Learned Proposal}
A learnable realization may use an amortized proposal network $\pi_{\phi}:\mathcal{C}\rightarrow\mathcal{S}$. For continuous point prediction, one compatible objective is
\begin{equation}
\mathcal{L}_{\text{prop}}(\phi)
=
\mathbb{E}_{(c,s)}\!\left[\lVert \pi_{\phi}(c)-s\rVert_2^2\right].
\end{equation}
Under the specific ambiguous distributions analyzed by Theorem~\ref{thm:loss_change_not_enough}, squared-loss point estimation targets the conditional mean. In a Settling realization, this proposal need not itself be valid; its role is to provide a fast, context-aligned initialization $s^{(0)}=\pi_{\phi}(c)$ that can subsequently be revised.

Other proposal objectives are possible. In particular, mixture-valued, set-valued, autoregressive, or explicitly mode-seeking proposals need not collapse to an arithmetic mean and are outside the failure statement of Theorem~\ref{thm:loss_change_not_enough}. A learned comparison should therefore include such baselines rather than treating all neural proposals as mean-seeking.

\subsection{Possible Learned Consistency Energy}
A learnable consistency critic $E_{\theta}:\mathcal{S}\times\mathcal{C}\rightarrow\mathbb{R}$ may be trained to rank valid states below invalid ones for a fixed context. One compatible margin objective is
\begin{equation}
\mathcal{L}_{\text{energy}}(\theta)
=
\mathbb{E}_{(c,s^+,s^-)}
\left[
\max\!\left(0, E_{\theta}(s^+,c)-E_{\theta}(s^-,c)+m\right)
\right],
\end{equation}
where $s^+\in\mathcal{V}(c)$, $s^-\notin\mathcal{V}(c)$, and $m>0$. Structured interpolations between valid modes are useful negatives only when an external validity rule or labeled data establishes that those interpolations are actually invalid. The formulation therefore does not assume that averages are universally incorrect; it asks the energy to represent the task-specific validity structure supplied by data or constraints.

\subsection{Training Coupling Is an Open Design Choice}
The proposal and consistency critic have distinct inferential roles, but this paper does not evaluate whether they should be trained independently, jointly, or with alternating objectives. We therefore avoid making a general claim that joint training necessarily helps or harms ambiguity resolution. What the operator requires at inference time is simply an initialization and a consistency landscape whose local minima correspond sufficiently well to valid interpretations.

\subsection{What the Theory Transfers to a Learned Realization}
The theoretical bridge from the analytic diagnostic to a learned realization is intentionally conditional. Theorem~\ref{thm:settling_convergence} requires only local regularity around a target equilibrium, while Lemma~\ref{lem:inexact-gradient} shows that monotone refinement survives bounded relative gradient error. Thus a learned critic need not recover the analytic energy pointwise; it must recover a locally useful vector field and stable basin geometry. This is a weaker and more testable requirement. Conversely, if learned gradients are poorly aligned, if the proposal falls outside recoverable basins, or if spurious stable points dominate, the learned realization should fail even though the analytic diagnostic succeeds. The paper therefore turns the learned-deployment question into explicit hypotheses that can be measured rather than assuming transfer by analogy.

\subsection{Controlled Analytic Instantiation}\label{subsec:analytic-instantiation}
The instantiation actually evaluated in Section~\ref{sec:experiments} uses no neural-network training. The proposal is the analytic conditional mean of two designated valid trajectory modes, and the consistency landscape is defined directly from obstacle geometry plus weak proposal anchoring and smoothness. This diagnostic design keeps every relevant failure and recovery mechanism inspectable: the mean is known, invalidity is explicit, the barrier geometry is known, and the refinement trajectory is reproducible. The resulting experiment tests whether equilibrium selection can repair an invalid aggregate under the stated landscape; it does not by itself establish learned high-dimensional generalization.

\section{Controlled Evidence and Diagnostics}\label{sec:experiments}
\subsection{Evidence Organization and Claim Scope}
The evidence blocks do not carry equal weight. Evidence~I is the quantitative reproducibility anchor: its source code, randomized-context evaluation, publication figures, and animation are intended for release in the accompanying GitHub reproducibility archive. The semantic and sensor-fusion panels are mechanism illustrations rather than benchmark evidence. The numerical ablations and geometric diagnostics are reproducible from the released analytic code. Table~\ref{tab:evidence-status} makes this separation explicit.

\begin{table}[t]
\centering
\small
\begin{tabular}{p{0.18\linewidth}p{0.24\linewidth}p{0.48\linewidth}}
\toprule
Evidence block & Status & Claim supported \\
\midrule
Geometric ambiguity & Reproducible quantitative & Conditional-mean invalidity and recovery under the declared analytic landscape and 100 randomized contexts. \\
Semantic ambiguity & Analytical / retained illustration & Qualitative compatibility of the operator interpretation with continuous degradation versus nonlinear branch selection; no re-estimated CLIP benchmark claim. \\
Sensor fusion & Retained illustration & Conceptual transfer of equilibrium selection to conflicting evidence; no quantitative cross-domain claim. \\
Ablations / diagnostics & Reproducible analytic code & Dependence on refinement, tie breaking, and numerical basin selection in the released geometric implementation. \\
\bottomrule
\end{tabular}
\caption{Evidence status and claim boundary. Only the geometric study and code-generated diagnostics are used as quantitative reproducible evidence.}
\label{tab:evidence-status}
\end{table}

The purpose throughout is mechanism isolation rather than benchmark ranking. In particular, the paper does not present a learned high-dimensional deployment; that limitation is discussed explicitly in Section~\ref{sec:limitations}.

\subsection{Evidence I: Geometric Ambiguity and Topological Validity}\label{subsec:exp-planning}
The geometric environment represents an A-to-B decision with non-convex validity. The state sequence contains 40 waypoints from $(-1,0)$ to $(1,0)$. In the nominal context, two circular obstacles of radius $0.25$ are centered at $(-0.4,0)$ and $(0.4,0)$. Two designated valid trajectories pass on opposite sides of the obstacles, while their arithmetic mean occupies the central invalid region. Figure~\ref{fig:geom-env} shows the environment and the validity geometry.

\begin{figure}[t]
\centering
\includegraphics[width=0.92\linewidth]{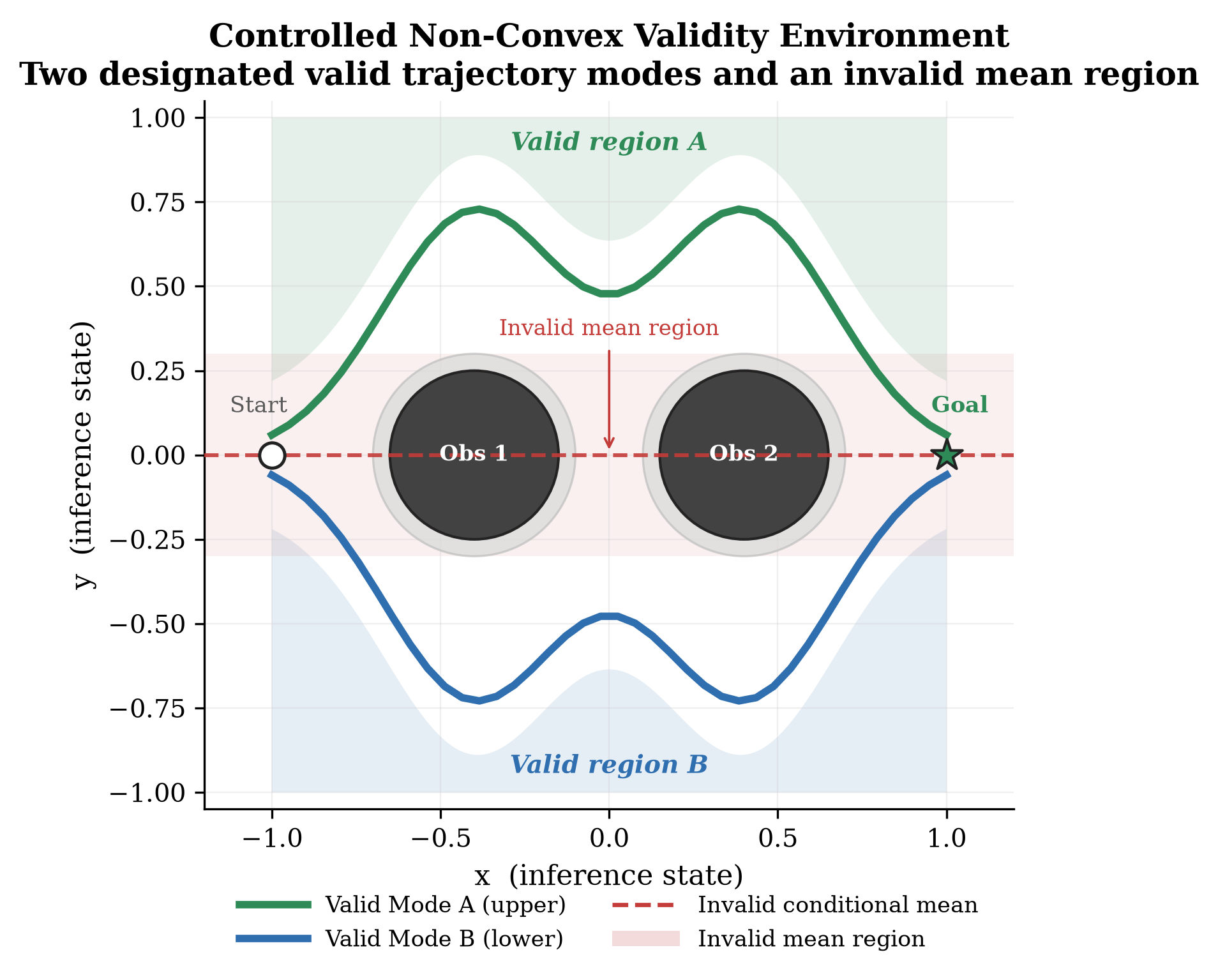}
\caption{Reproducible geometric environment used for Evidence~I. The two designated A-to-B modes are collision free, while the central mean trajectory intersects the obstacle region.}
\label{fig:geom-env}
\end{figure}

A trajectory is counted as collision free only when every waypoint clears both obstacle cores by at least $0.02$. For the 100-context evaluation, the vertical coordinates of the two obstacle centers are independently jittered uniformly within $[-0.08,0.08]$. The same contexts are presented to all methods.

\subsubsection{Operator baselines and inference budgets}
The comparison is deliberately operator-level and should not be read as a benchmark of complete Transformer, diffusion, or EBM systems.
\begin{itemize}
\item \textbf{Mean-seeking / attention-like aggregate.} The output is the analytical arithmetic mean of the two designated valid modes. In the symmetric context this is the exact conditional mean associated with squared-loss point estimation.
\item \textbf{Stochastic denoising.} The trajectory is initialized from Gaussian noise and iteratively smoothed while obstacle/end-point gradients and decaying noise are applied. The 100-context evaluation uses 100 refinement steps. This is a stylized stochastic refinement baseline, not a claim about every diffusion architecture.
\item \textbf{Direct energy descent.} Gradient descent is performed from an unguided random initialization using obstacle, end-point, and light smoothness terms. The 100-context evaluation uses 180 steps. This baseline tests whether energy descent alone reproduces the behavior of proposal-grounded Settling.
\item \textbf{Settling.} The proposal is the analytical mean, followed by a one-time symmetry-breaking initialization perturbation and deterministic refinement of a seven-control-point spline representation. The evaluation uses 120 settling steps. The released implementation uses obstacle weight $\lambda_{\mathrm{obs}}=80$ and proposal-anchor weight $\lambda_{\mathrm{anchor}}=0.05$.
\end{itemize}
For the released quantitative script, the symmetry-breaking standard deviation is $\sigma=0.30$. The dynamics visualization in Fig.~\ref{fig:settling-dynamics} uses $\sigma=0.025$ to make departure from the conditional-mean neighborhood visually explicit. In both cases, noise is applied only at initialization; Settling does not inject noise during iterative refinement.

\subsubsection{Metrics and quantitative results}
We report collision-free success, collision rate, mean squared second-difference smoothness, and mean minimum obstacle clearance. Smoothness is a diagnostic of geometric regularity, not a task reward. Minimum clearance is measured relative to the physical obstacle radius; negative values indicate penetration.

\begin{table}[t]
\centering
\begin{tabular}{lrrrr}
\toprule
Method & Success & Collision & Smooth. & Clearance\\
\midrule
Mean-seeking aggregate & 0\% & 100\% & 0.00000 & -0.235\\
Stochastic denoising & 100\% & 0\% & 0.02631 & 0.045\\
Direct energy descent & 0\% & 100\% & 0.03022 & -0.178\\
\textbf{Settling} & \textbf{99\%} & \textbf{1\%} & \textbf{0.00092} & 0.027\\
\bottomrule
\end{tabular}
\caption{Quantitative comparison on 100 randomized obstacle contexts using the current released implementation. Success is a binary complete-trajectory criterion requiring every waypoint to satisfy the declared clearance rule. Settling succeeds in 99/100 contexts; the single failure misses the $0.02$ clearance threshold.}
\label{tab:quantitative-results}
\end{table}

The mean-seeking aggregate fails in every evaluated context because it remains in the obstacle-intersecting central region. Stochastic denoising succeeds in 100/100 contexts and Settling in 99/100. The single Settling failure occurs at seed 25, where the minimum clearance is $0.01674$, slightly below the declared $0.02$ threshold. Settling nevertheless produces much lower second-difference roughness than the stochastic baseline ($0.00092$ versus $0.02631$) and does not rely on stochastic refinement after initialization. Direct energy descent fails in the evaluated configuration, showing that an energy landscape without proposal-grounded equilibrium formation is not sufficient in this diagnostic.

\subsubsection{Robustness to context perturbation and tie-breaking amplitude}
To test whether the nominal result depends narrowly on the declared obstacle jitter or on a single symmetry-breaking amplitude, we ran a 1,200-run Settling sensitivity sweep: six obstacle-jitter ranges with 100 seeds per point at fixed $\sigma=0.30$, and six initialization perturbation levels with 100 seeds per point at the nominal jitter range $0.08$. The update budget (120 steps), collision criterion, obstacle geometry, and seed construction were held fixed; no parameter was retuned at individual sweep points.

Figure~\ref{fig:robustness-sweeps} reports the resulting success rates. At obstacle jitter $0.00$ and $0.04$, Settling succeeds in 100/100 trials; at $0.08$ and $0.12$, it succeeds in 99/100; at $0.16$, in 97/100; and at $0.20$, in 98/100. The median minimum clearance remains approximately $0.0495$ throughout this sweep, and the fifth-percentile clearance remains above $0.0347$, although a small lower tail crosses the $0.02$ success threshold at the more difficult settings. At the nominal jitter range, initialization perturbations $\sigma=0.05$, $0.10$, and $0.20$ each yield 100/100 success; $\sigma=0.30$ yields 99/100; and $\sigma=0.40$ and $0.50$ each yield 94/100. Thus the mechanism is robust over a broad neighborhood of the nominal setting, but excessively large tie-breaking perturbations measurably reduce reliability. This sensitivity is useful rather than hidden: it identifies initialization amplitude as a genuine operating parameter of the current implementation.

\begin{figure}[!htbp]
\centering
\includegraphics[width=0.96\linewidth]{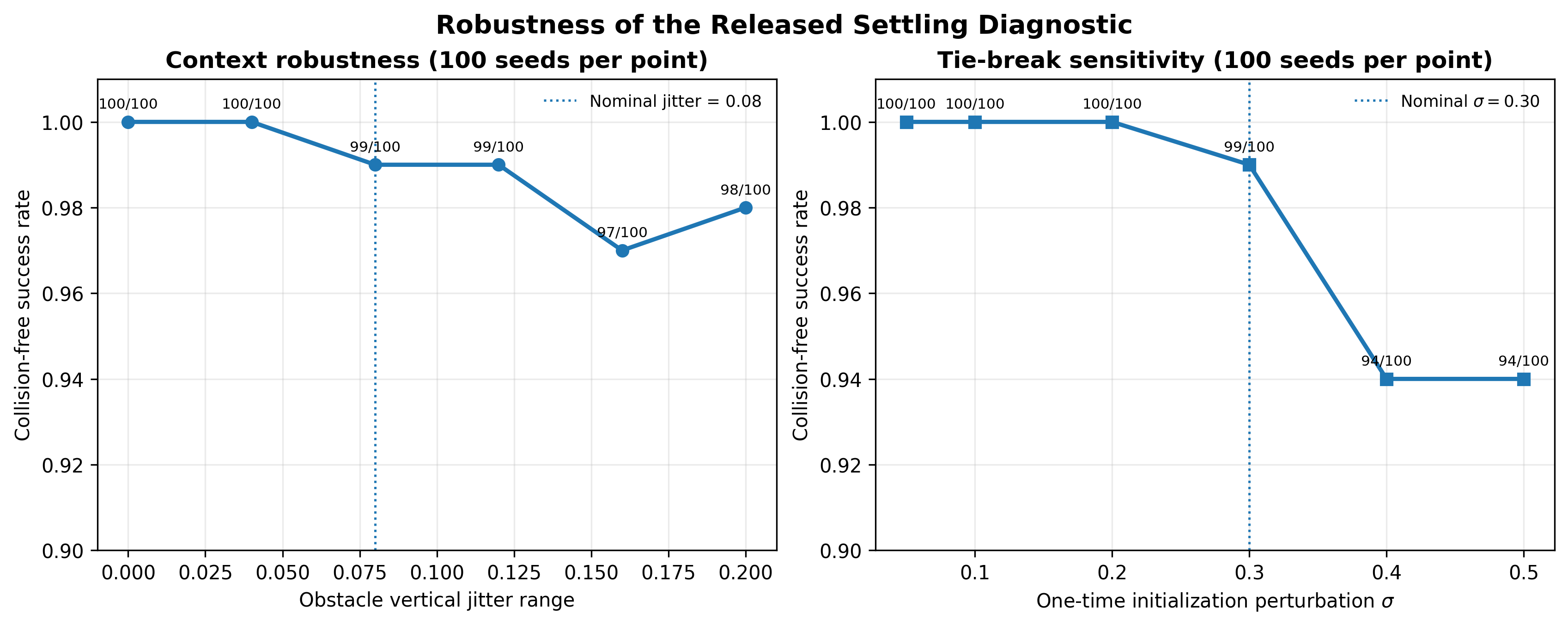}
\caption{Robustness sweeps for the released Settling diagnostic, with 100 seeds per sweep point. Left: collision-free success under increasing obstacle vertical jitter, with the nominal paper setting at $0.08$. Right: sensitivity to the one-time symmetry-breaking perturbation $\sigma$ at fixed jitter $0.08$. The current nominal setting $\sigma=0.30$ reproduces the 99/100 baseline result. Larger perturbations expose a measurable failure regime rather than being tuned away.}
\label{fig:robustness-sweeps}
\end{figure}
\FloatBarrier

\begin{figure}[t]
\centering
\includegraphics[width=0.94\linewidth]{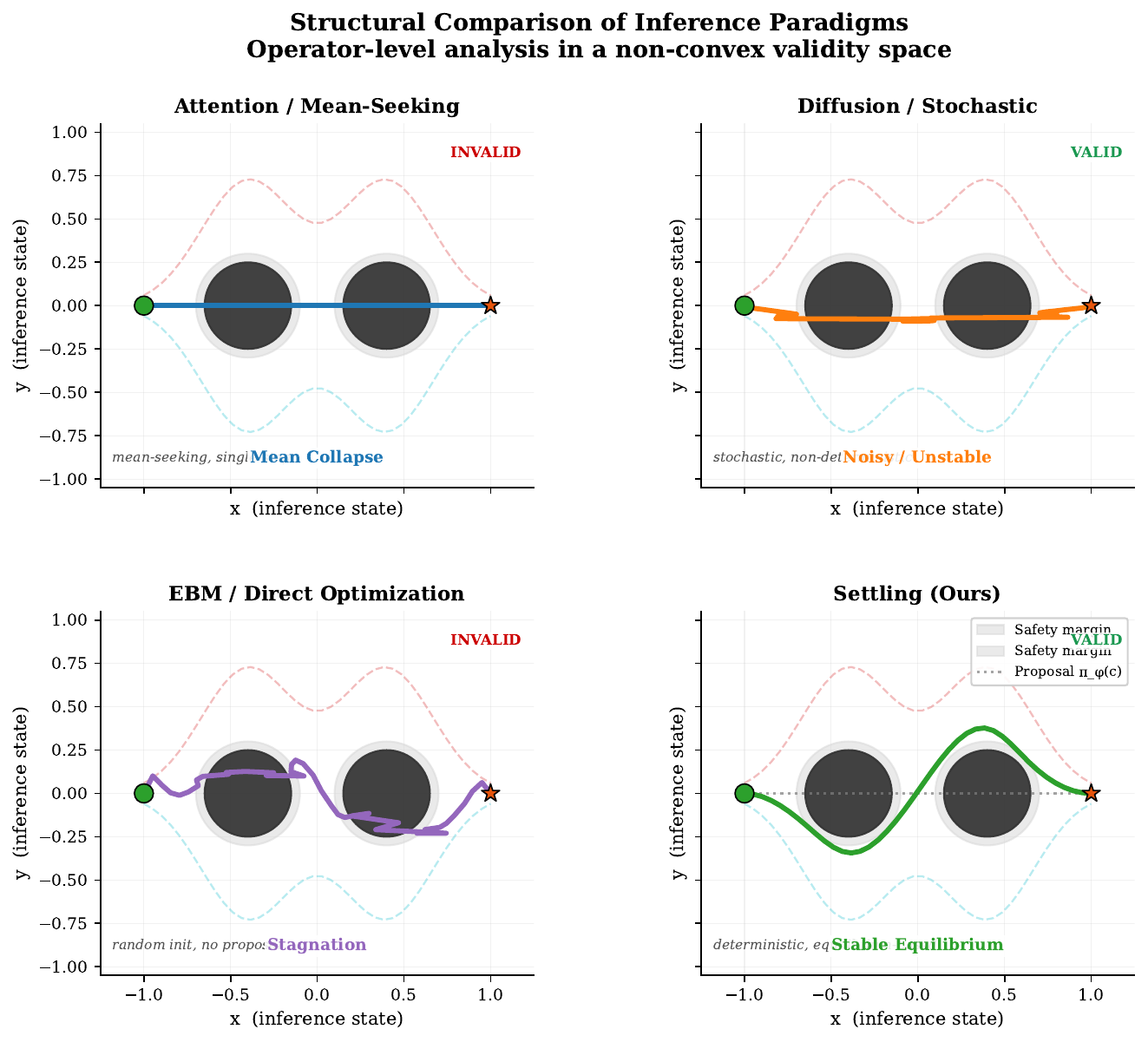}
\caption{Four-operator comparison in the nominal geometric context. The mean-seeking aggregate is invalid; the stochastic baseline reaches a valid but irregular trajectory; direct energy descent remains invalid for the shown initialization; Settling reaches a valid, smooth equilibrium.}
\label{fig:all-methods}
\end{figure}

\subsubsection{Inference-time trajectory}
Figure~\ref{fig:settling-dynamics} shows the internal trajectory during refinement. The initialized state lies near the invalid mean, remains invalid during early iterations, and then moves into a collision-free basin. The supplied GIF shows the same transition continuously.

\begin{figure}[t]
\centering
\includegraphics[width=0.98\linewidth]{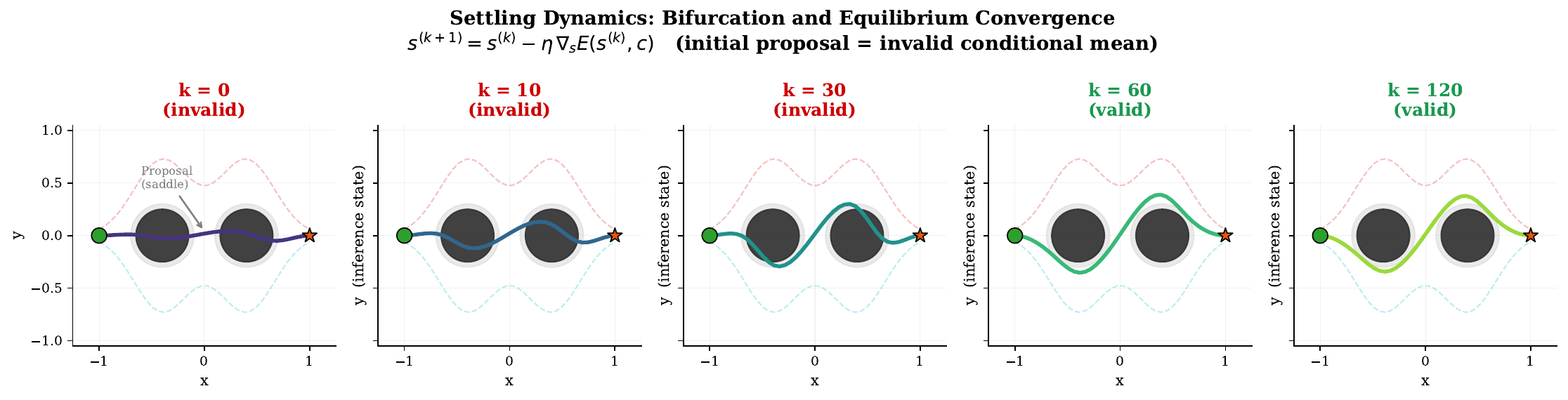}
\caption{Settling dynamics for the nominal context. Snapshots show departure from the invalid mean neighborhood and convergence to a collision-free equilibrium. The visualization uses 120 steps, learning rate $0.006$, and initialization perturbation $\sigma=0.025$.}
\label{fig:settling-dynamics}
\end{figure}

\subsection{Illustration II: Semantic Ambiguity and Nonlinear Symmetry-Breaking Reference}\label{subsec:exp-semantic}
The semantic setting originates from the controlled vision--language example based on CLIP embeddings~\citep{radford2021clip}. Ambiguity is represented by increasing overlap between candidate interpretations. The original manuscript used a candidate set containing a ground-truth caption, hard semantic distractors, and soft distractors, with static similarity-based aggregation contrasted against commitment-driven Settling dynamics. Because the original semantic sweep-generation script and underlying sweep values are not used as quantitative evidence in this revision, the nonlinear-transition panel below is explicitly treated as an analytical reference rather than as a measured CLIP result.

\subsubsection{Analytical nonlinear-transition reference}
Figure~\ref{fig:semantic-phase} is a theory-guided schematic, not a fitted semantic benchmark. The dashed reference decreases continuously with overlap. The Settling reference uses the normalized square-root branch
\begin{equation}
m(\lambda)=m_0+\beta\sqrt{[\lambda-\lambda_c]_+},
\qquad [x]_+=\max(x,0),
\end{equation}
which is the local branch shape of a simple supercritical pitchfork normal form after loss of stability of a symmetric state. Here it is used only to visualize a possible nonlinear branch-selection mechanism. The transition location $\lambda_c\approx0.72$, scale $\beta$, and plotted amplitudes are illustrative; they are not estimated from CLIP data, do not imply thermodynamic criticality, and are not used as empirical evidence.

\begin{figure}[t]
\centering
\includegraphics[width=0.90\linewidth]{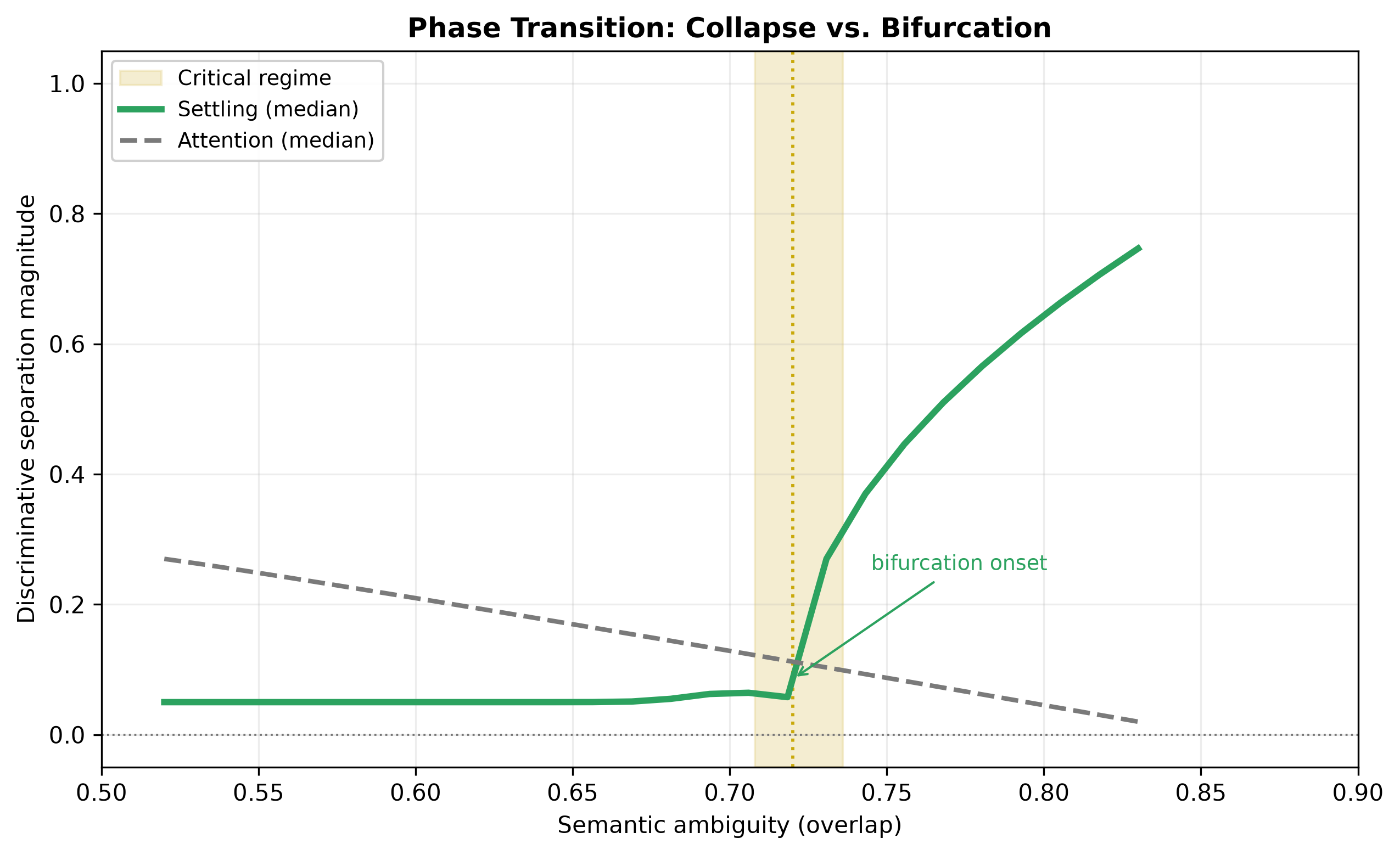}
\caption{Analytical nonlinear-transition reference for semantic ambiguity. The dashed curve represents continuous degradation of a mean-seeking separation measure; the green curve is a normalized square-root reference branch used to illustrate symmetry breaking after an illustrative transition point $\lambda_c$. The curves are not fitted CLIP statistics and no universal criticality claim is made.}
\label{fig:semantic-phase}
\end{figure}

\subsubsection{Instance-level semantic dynamics}
Figure~\ref{fig:semantic-instance} retains the original instance-level semantic illustration. The first stage represents a calibrated but ambiguous state, while the second stage amplifies residual asymmetry until one hypothesis dominates. Its role is structural: it makes the semantic example parallel to the geometric one by treating the ambiguous internal state as revisable rather than terminal. Because the underlying generation script is not part of the released package, this figure is not included in the quantitative reproducibility claims.

\begin{figure}[t]
\centering
\includegraphics[width=0.96\linewidth]{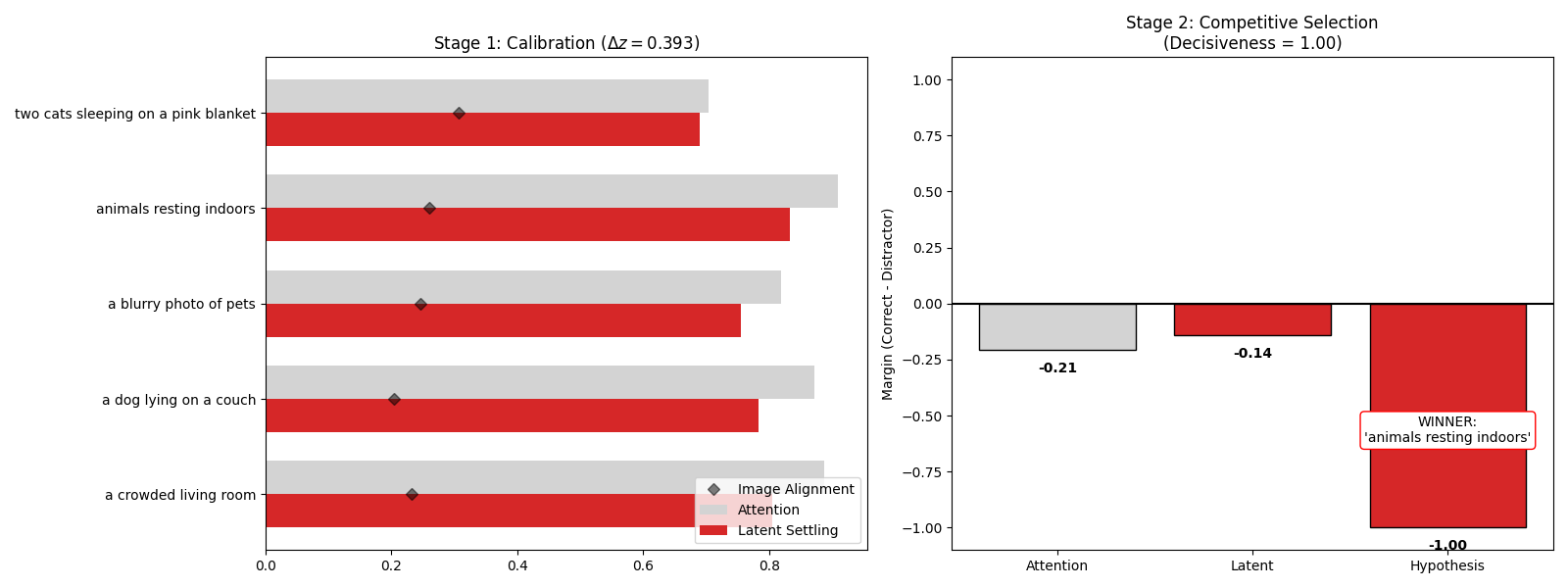}
\caption{Retained two-stage semantic illustration. A calibrated ambiguous state is followed by competitive equilibrium selection that commits to one interpretation. This panel is used as a mechanism illustration rather than as quantitative benchmark evidence.}
\label{fig:semantic-instance}
\end{figure}

\subsubsection{Risk--commitment illustration}
The retained risk--commitment panel in Fig.~\ref{fig:risk-commitment} visualizes the intended effect of varying commitment strength $\alpha$: stronger competition moves the system through operating points with different decisiveness and confidence-error behavior rather than leaving it at a static aggregate. As with the instance-level panel, we treat this figure as a controlled illustration and do not infer a universal calibration law from it.

\begin{figure}[t]
\centering
\includegraphics[width=0.86\linewidth]{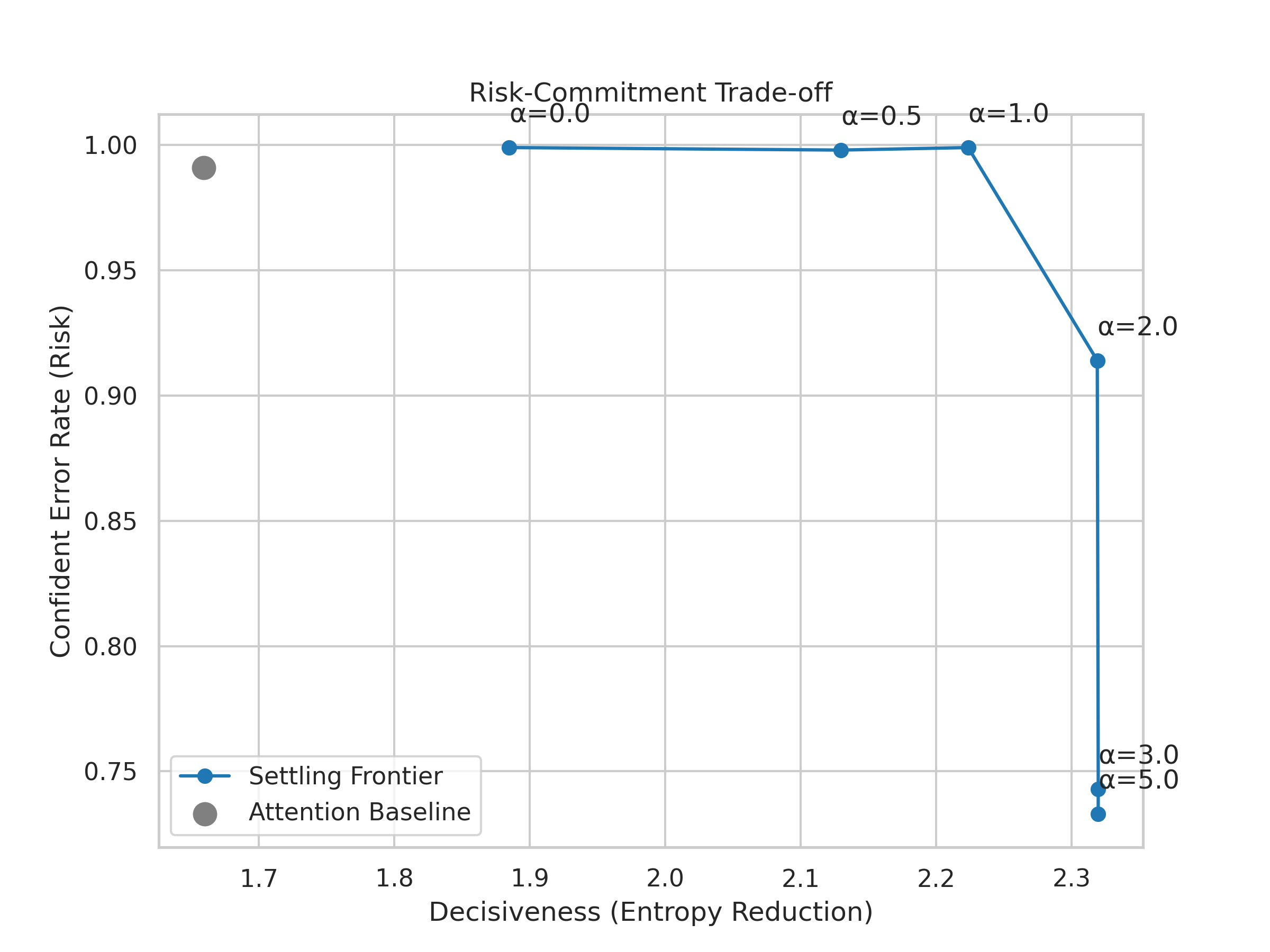}
\caption{Retained risk--commitment illustration. Varying the commitment strength $\alpha$ exposes operating regimes that are unavailable to a static aggregate in this controlled example; the curve is not interpreted as a universal calibration frontier.}
\label{fig:risk-commitment}
\end{figure}

\subsection{Illustration III: Conflicting-Evidence Sensor Fusion}\label{subsec:exp-sensor}
The third study retains the original sensor-fusion example. Two conflicting sensory hypotheses induce a symmetric high-energy state near their average. Standard aggregation terminates at this compromise, whereas Settling treats it as an unstable configuration and evolves toward a structurally consistent basin. Figure~\ref{fig:sensor-fusion} visualizes the one-dimensional energy field and the settling trajectory.

\begin{figure}[t]
\centering
\includegraphics[width=0.84\linewidth]{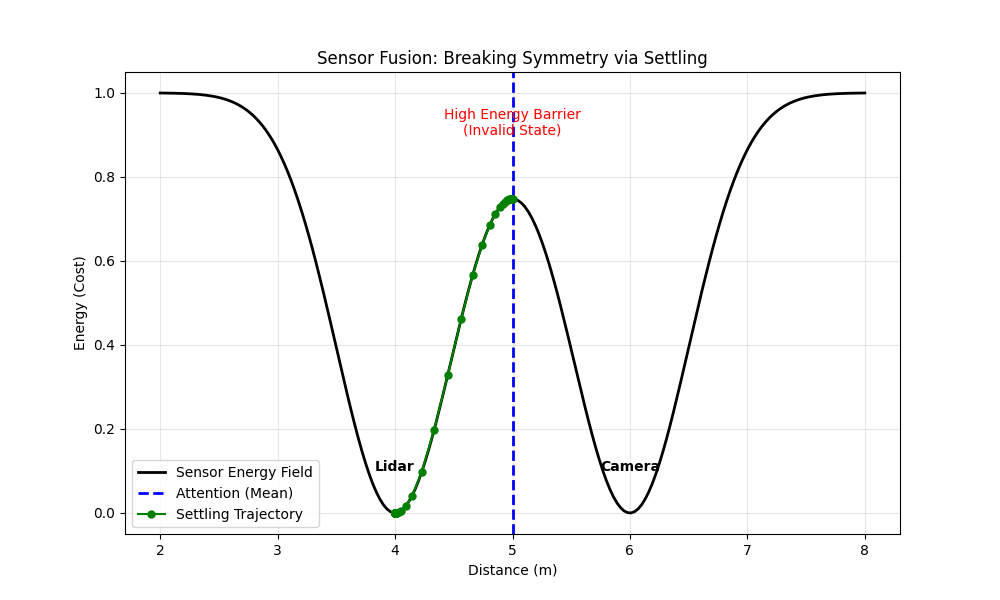}
\caption{Original sensor-fusion demonstration. Conflicting hypotheses create a high-energy compromise near the mean; Settling breaks the symmetry and converges toward one stable sensory-consistent equilibrium.}
\label{fig:sensor-fusion}
\end{figure}

The value of this example is conceptual rather than benchmark-scale: it demonstrates that equilibrium selection can be applied to conflicting evidence without requiring the output to be an arithmetic compromise. It also exposes the same dependence on basin geometry that appears in the theoretical formulation.

\subsection{Evidence IV: Ablation and Numerical Dynamics Diagnostics}\label{subsec:ablation}
Two original ablations isolate mechanisms that are necessary in the controlled symmetric setting.
\begin{enumerate}
\item \textbf{Frozen dynamics.} Disabling Settling reduces the method to a static proposal and restores conditional mean collapse.
\item \textbf{No symmetry-breaking perturbation.} In an exactly symmetric context, initializing exactly at the saddle can leave the system stationary. A deterministic or randomized tie-breaking perturbation is therefore required only for exact symmetry.
\end{enumerate}
These ablations clarify the conditional use of ``deterministic'': once the initialized state is fixed, the Settling updates are deterministic; if a fresh random tie-break is drawn, the complete end-to-end mapping is deterministic only conditional on that draw.

The code package also supplies two additional diagnostics. Figure~\ref{fig:energy-convergence} tracks numerical energy-like quantities for the four operator realizations. It is reported as an implementation diagnostic and not as a proof of Lemma~\ref{lem:lyapunov}, because the practical code includes clipping, endpoint constraints, spline parameterization, and method-specific updates. Figure~\ref{fig:geom-sweep} shows the controlled ambiguity sweep generated by the released geometric code. The plotted Settling commitment is a binary single-run indicator at each ambiguity value, so switches between 0 and 1 reflect seed-sensitive basin selection after the one-time tie break rather than an estimated commitment probability. We therefore use the panel only as a qualitative regime diagnostic and do not infer a universal critical threshold from it.

\begin{figure}[t]
\centering
\includegraphics[width=0.90\linewidth]{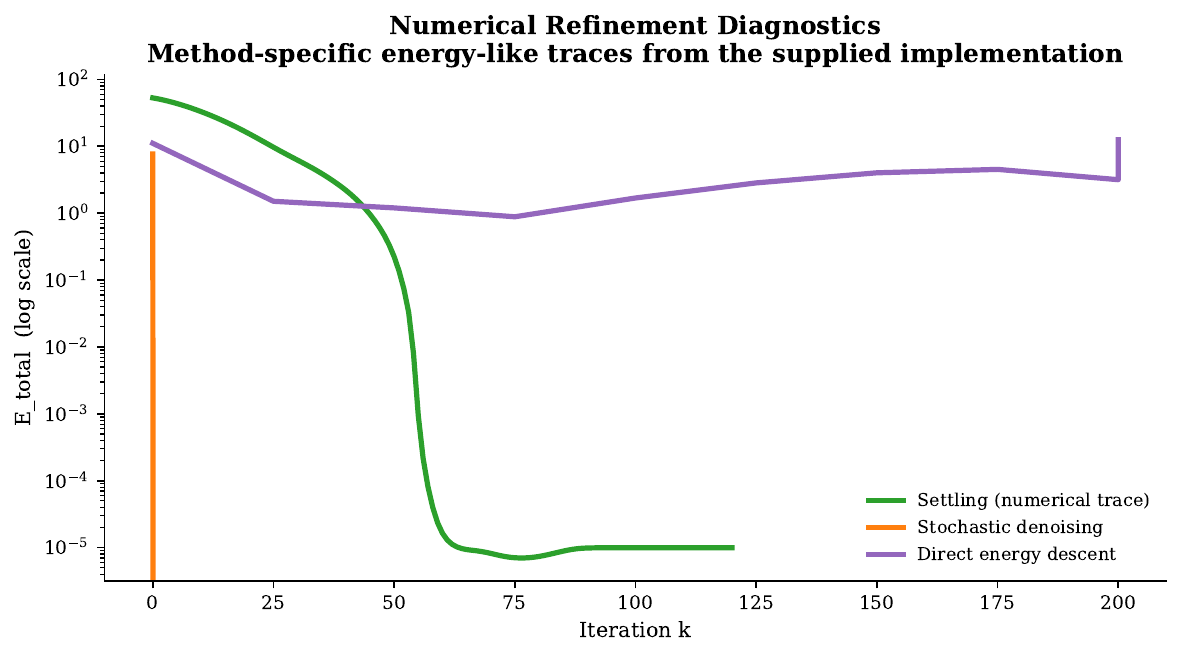}
\caption{Numerical convergence diagnostic from the supplied code. The curves visualize method-specific refinement behavior; they are not used as a formal certificate of the smooth-gradient monotonicity theorem.}
\label{fig:energy-convergence}
\end{figure}

\begin{figure}[t]
\centering
\includegraphics[width=0.88\linewidth]{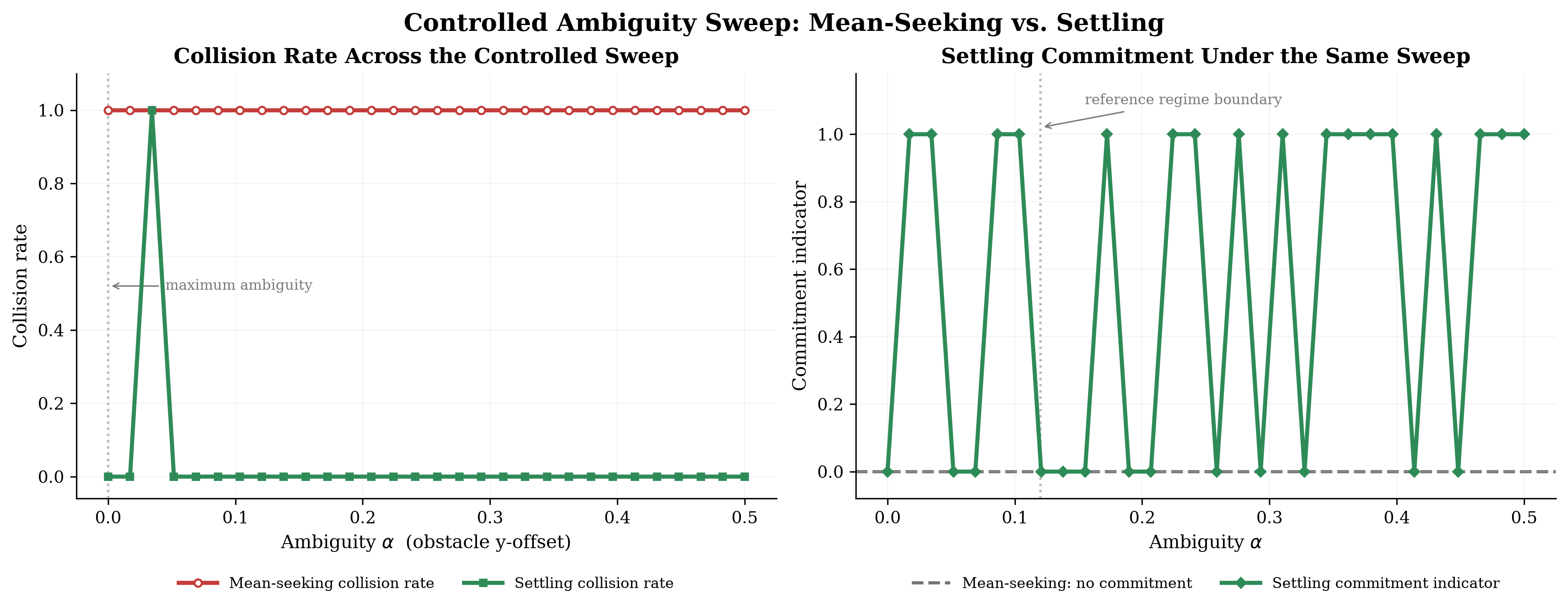}
\caption{Controlled ambiguity sweep generated by the released geometric code. The left panel compares collision outcomes; the right panel reports a binary single-run Settling commitment indicator at each ambiguity value. The panel is a qualitative diagnostic of basin selection, not a probability estimate, and no claim is made that the displayed reference boundary is a universal transition threshold.}
\label{fig:geom-sweep}
\end{figure}

A reduced two-dimensional consistency landscape is provided in
Appendix~\ref{app:additional-geometry} (Fig.~\ref{fig:energy-landscape-app}).
It is included to provide geometric intuition for the local basin structure;
the actual inference state in Evidence~I is a high-dimensional trajectory
represented by multiple control points.

\section{Discussion, Limitations, and Claim Boundaries}\label{sec:limitations}
The evidence supports a common operator-level interpretation: a mean-seeking point estimate can terminate in an invalid compromise when valid alternatives occupy a non-convex set, whereas a proposal-grounded consistency process can treat that compromise as an intermediate state and evolve toward a stable equilibrium. The geometric study is the quantitative anchor because its validity criterion, randomized contexts, baselines, and evaluation code are explicit and reproducible. The semantic and sensor-fusion panels broaden only the conceptual scope; the semantic nonlinear-transition panel is analytical rather than empirical. This separation prevents illustrative cross-domain figures from carrying stronger claims than the released evidence supports.

\paragraph{Transfer to learned high-dimensional systems.}
The present study does not include an end-to-end learned high-dimensional deployment. The inexact-gradient result narrows what must transfer, namely local gradient alignment and stable basin geometry, but it does not establish that a learned critic will satisfy those conditions. Representation error, negative-state construction, spurious equilibria, basin coverage, and inference cost therefore remain empirical questions. This boundary is deliberate: the contribution evaluated here is the operator-level failure mode, equilibrium-selection mechanism, and its robustness under a transparent consistency landscape; learning a useful high-dimensional critic is a separate empirical hypothesis rather than an assumption hidden inside the present claims. A decisive learned follow-up should evaluate at least one real high-dimensional data set, train both proposal and consistency components, compare against mode-preserving as well as mean-seeking baselines, report validity and task quality jointly, and measure inference-time compute and sensitivity to initialization.

Several additional boundaries are essential. First, Evidence~I does not establish superiority over state-of-the-art trajectory planners, diffusion policies, learned EBMs, mixture predictors, or set-valued predictors. The stochastic denoising and direct-descent implementations are stylized operator baselines selected to expose different inference geometries. Illustration~II and Illustration~III are not reproducible benchmark evidence and are not used to support quantitative cross-domain superiority.

Second, determinism is conditional on the initialized state. Exact symmetric saddles require a tie-breaking rule. The released geometric implementation uses a one-time random perturbation, after which all refinement steps are deterministic. The new sensitivity sweep shows that this perturbation is not innocuous: success is 100/100 for $\sigma\leq0.20$ in the tested nominal context distribution, 99/100 at the reported $\sigma=0.30$, and 94/100 at $\sigma=0.40$ and $0.50$. A context-derived fixed perturbation or a validation-selected amplitude could reduce this source of variability, but such a design change is not retrofitted into the present test set.

Third, the monotonicity result in Section~\ref{sec:theory_settling} applies to the unprojected gradient update under an $L$-smooth energy and a step size satisfying the stated bound. Practical implementations may use clipping, endpoint constraints, numerical gradients, and reduced parameterizations. The theory should therefore be read as a property of the idealized operator under its assumptions rather than a blanket certificate for every implementation detail.

Fourth, the quality of the selected equilibrium depends on the consistency landscape. In learned deployments, misspecified energies can create spurious minima or semantically incorrect basins. This dependence is not hidden by the framework; it is the main modeling responsibility shifted from averaging to explicit validity representation.

Finally, Settling introduces inference-time computation. The exact budget is domain dependent rather than a universal 40--100-step rule. In the released geometric evaluation the budgets are stated explicitly for each method; other domains should report their own stopping conditions, iteration counts, and wall-clock costs. Adaptive stopping and learned basin shaping are therefore important directions for future work.

\section{Conclusion}\label{sec:conclusion}
This paper formulates \emph{conditional mean collapse} as a concrete failure of squared-loss point inference in non-convex validity sets and develops Settling as an equilibrium-selection operator that separates proposal generation from explicit consistency evaluation and test-time refinement. The contribution is not gradient descent or fixed-point computation in isolation; it is the inference-level decomposition that treats an invalid aggregate as a revisable internal hypothesis.

In the fully reproducible 100-context geometric diagnostic, the analytical mean-seeking baseline collides in every case, the stylized stochastic denoising baseline succeeds in 100/100 contexts, and Settling succeeds in 99/100 while yielding much lower trajectory roughness and using no stochasticity during refinement after initialization. The 1,200-run robustness study maintains 97--100\% success across obstacle-jitter ranges up to $0.20$ and 94--100\% across initialization perturbations from $0.05$ to $0.50$, while exposing a measurable reliability loss at larger tie-breaking amplitudes. Together with the local-convergence and inexact-gradient results, the study establishes a transparent operator-level proof of concept and identifies concrete conditions that a learned realization must satisfy. The semantic and sensor-fusion panels remain mechanism illustrations rather than quantitative evidence. The next empirical test is therefore sharply defined: determine whether learned proposals and consistency critics preserve useful basin geometry and gradient alignment on realistic data while remaining competitive in task quality and inference cost.

\section*{Data and Code Availability}

The code used for the quantitative geometric study and 1,200-run robustness sweep, the scripts used to generate the reported code-based figures and diagnostics, and the supplementary \texttt{settling\_animation.gif} will be archived in the author's public GitHub repository. The repository URL is: 

\texttt{https://github.com/LyesSaadSaoud/Settling-Equilibrium-Inference}.

\section*{Declaration of Generative AI Use}
Generative AI was used only to refine the English language and improve readability. The author reviewed and approved the final manuscript and takes full responsibility for its content.
\section*{Funding}
This research received no external funding.

\appendix
\section{Proof of Conditional Mean Collapse}\label{app:proofs}
\begin{proof}[Proof of Theorem~\ref{thm:loss_change_not_enough}]
For a fixed context $c$, let
$p(y\mid c)=\lambda\delta(y-y_1)+(1-\lambda)\delta(y-y_2)$.
The conditional squared-loss risk for a point prediction $\hat y$ is
\[
R(\hat y\mid c)=\lambda\|y_1-\hat y\|_2^2+(1-\lambda)\|y_2-\hat y\|_2^2.
\]
Differentiating with respect to $\hat y$ gives
\[
\nabla_{\hat y}R=2\lambda(\hat y-y_1)+2(1-\lambda)(\hat y-y_2).
\]
Setting the gradient to zero yields the unique minimizer
\[
\hat y^*=\lambda y_1+(1-\lambda)y_2=\bar y.
\]
By assumption $\bar y\notin\mathcal V(c)$, so the Bayes-optimal squared-loss point predictor is invalid for that context. Increasing model capacity while preserving the same point-estimation objective cannot alter the Bayes solution. This establishes the stated structural mismatch.
\end{proof}

\section{Consistency Monotonicity}\label{app:monotonicity}
\begin{proof}[Proof of Lemma~\ref{lem:lyapunov}]
For an $L$-smooth function, the descent lemma gives
\[
\mathcal E(s-\eta\nabla\mathcal E(s),c)\leq \mathcal E(s,c)-\eta\left(1-\frac{L\eta}{2}\right)\|\nabla\mathcal E(s,c)\|_2^2.
\]
For $0<\eta<2/L$, the coefficient multiplying the gradient norm is positive, so the energy is non-increasing and decreases strictly whenever the gradient is nonzero. This result is stated for the unprojected idealized update; projected or clipped implementations require the corresponding additional assumptions.
\end{proof}

\section{Inexact-Gradient Consistency Descent}\label{app:inexact-gradient}
\begin{proof}[Proof of Lemma~\ref{lem:inexact-gradient}]
Let $g=\nabla_s\mathcal E(s,c)$ and $\tilde g=g+e$. By $L$-smoothness,
\[
\mathcal E(s-\eta\tilde g,c)
\leq \mathcal E(s,c)-\eta g^\top\tilde g+\frac{L\eta^2}{2}\lVert\tilde g\rVert_2^2.
\]
The relative-error assumption gives
\[
g^\top\tilde g=\lVert g\rVert_2^2+g^\top e
\geq (1-\delta)\lVert g\rVert_2^2,
\]
and
\[
\lVert\tilde g\rVert_2\leq(1+\delta)\lVert g\rVert_2.
\]
Therefore
\[
\mathcal E(s-\eta\tilde g,c)
\leq \mathcal E(s,c)
-\eta\left[(1-\delta)-\frac{L\eta}{2}(1+\delta)^2\right]\lVert g\rVert_2^2.
\]
The bracket is positive when
$0<\eta<2(1-\delta)/(L(1+\delta)^2)$, giving strict decrease whenever $g\neq0$.
\end{proof}

\section{Local Convergence}\label{app:local-convergence}
\begin{proof}[Proof of Theorem~\ref{thm:settling_convergence}]
On the forward-invariant neighborhood $U$, $\mathcal E(\cdot,c)$ is $\mu$-strongly convex and $L$-smooth, so $s^\star$ is the unique minimizer of $\mathcal E$ in $U$. For gradient descent with $0<\eta\leq1/L$, the standard strong-convexity contraction gives
\[
\lVert s^{(k+1)}-s^\star\rVert_2
\leq (1-\eta\mu)\lVert s^{(k)}-s^\star\rVert_2.
\]
Because $U$ is forward invariant, every iterate remains in the region where the assumptions hold. Iterating the inequality yields
\[
\lVert s^{(k)}-s^\star\rVert_2
\leq (1-\eta\mu)^k\lVert s^{(0)}-s^\star\rVert_2\rightarrow0.
\]
Hence the Settling iterates converge linearly to the valid stationary point $s^\star$.
\end{proof}

\section{Reference Inference Operator}\label{app:algorithm}
\begin{algorithm2e}[H]
\caption{Settling inference via equilibrium relaxation}\label{alg:settling}
\KwIn{Context $c$; proposal $\pi_\phi$; consistency energy $\mathcal E$; step size $\eta$; settling budget $K$; optional constraint set $\Omega$; initialization tie-break $\epsilon$}
$s^{(0)}\leftarrow\pi_\phi(c)+\epsilon$\;
\For{$k\leftarrow0$ \KwTo $K-1$}{
$g^{(k)}\leftarrow\nabla_s\mathcal E(s^{(k)},c)$\;
$\tilde s^{(k+1)}\leftarrow s^{(k)}-\eta g^{(k)}$\;
$ s^{(k+1)}\leftarrow\mathrm{Proj}_{\Omega}(\tilde s^{(k+1)})$ if projection is required; otherwise $s^{(k+1)}\leftarrow\tilde s^{(k+1)}$\;
}
\KwOut{$s^{(K)}$}
\end{algorithm2e}

The perturbation $\epsilon$ is zero for asymmetric contexts unless a deterministic or randomized tie-break is needed to escape an exact saddle. If a random tie-break is used, the mapping is deterministic only conditional on that initialized state. The released controlled implementation optimizes a low-dimensional spline representation and hard-pins the start and goal rather than using a general projection operator.

\section{Additional Geometric Visualization}
\label{app:additional-geometry}
Figure~\ref{fig:energy-landscape-app} is the reduced two-dimensional consistency-landscape visualization supplied with the code package. It is included in the appendix for intuition only; the actual inference state in Evidence~I is a trajectory represented by multiple control points.
\begin{figure}[t]
\centering
\includegraphics[width=0.86\linewidth]{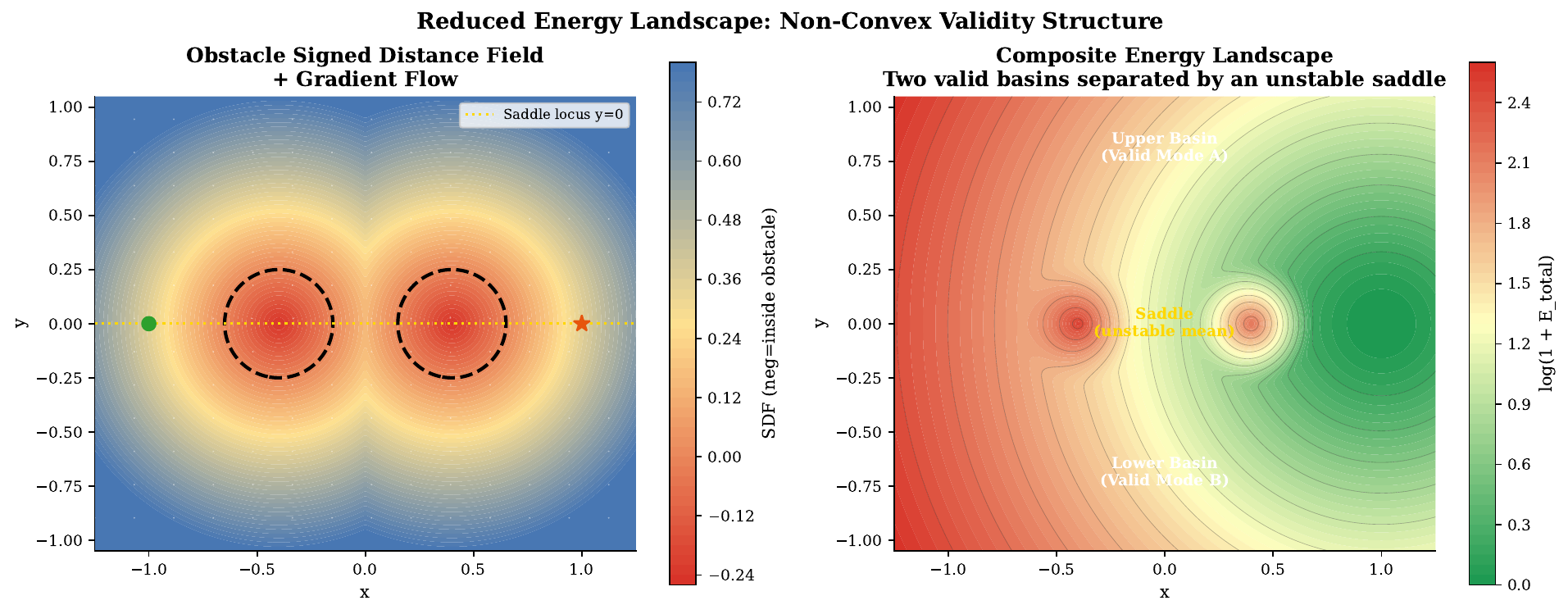}
\caption{Reduced geometric consistency landscape used for intuition. This slice is not the full trajectory-state energy and is not used as quantitative evidence.}
\label{fig:energy-landscape-app}
\end{figure}

\end{document}